\documentclass{article}
\usepackage[T1]{fontenc}
\usepackage{lmodern}
\usepackage{iclr2027_conference}

\usepackage{amsmath,amsfonts,bm}

\def\eqref#1{equation~\ref{#1}}

\def\plaineqref#1{\ref{#1}}

\def\1{\bm{1}}

\DeclareMathAlphabet{\mathsfit}{\encodingdefault}{\sfdefault}{m}{sl}
\SetMathAlphabet{\mathsfit}{bold}{\encodingdefault}{\sfdefault}{bx}{n}

\newcommand{\E}{\mathbb{E}}

\newcommand{\R}{\mathbb{R}}

\newcommand{\softmax}{\mathrm{softmax}}

\newcommand{\Var}{\mathrm{Var}}

\newcommand{\Cov}{\mathrm{Cov}}

\usepackage[colorlinks=true]{hyperref}
\usepackage{url}
\usepackage{graphicx}
\usepackage{capt-of}
\usepackage{wrapfig}
\usepackage{needspace}
\usepackage{amsmath,amssymb,amsthm}
\usepackage{booktabs}
\usepackage{xcolor}
\definecolor{citeblue}{RGB}{31,119,180}
\hypersetup{citecolor=citeblue,linkcolor=citeblue,urlcolor=citeblue}

\definecolor{adaptpinkbg}{RGB}{253,241,245}    
\definecolor{adaptpinkframe}{RGB}{174,91,116}  

\usepackage{enumitem}
\usepackage{algorithm}
\usepackage[noend]{algpseudocode}
\usepackage{fontawesome}   

\newtheoremstyle{paperthm}
  {7pt plus 2pt minus 1pt}{7pt plus 2pt minus 1pt}
  {\itshape}{}{\bfseries}{.}{0.6em}{}
\newtheoremstyle{paperdef}
  {7pt plus 2pt minus 1pt}{7pt plus 2pt minus 1pt}
  {\normalfont}{}{\bfseries}{.}{0.6em}{}
\theoremstyle{paperthm}
\newtheorem{theorem}{Theorem}
\newtheorem{proposition}[theorem]{Proposition}
\newtheorem{corollary}[theorem]{Corollary}
\newtheorem{lemma}[theorem]{Lemma}
\theoremstyle{paperdef}

\theoremstyle{remark}

\newcommand{\Prob}{\mathbb{P}}
\newcommand{\ind}[1]{\mathbf{1}\{#1\}}

\newcommand{\adaptationnote}[1]{%
  \par
  \Needspace{7\baselineskip}%
  \smallskip
  \begingroup
  \setlength{\fboxsep}{5pt}%
  \setlength{\fboxrule}{0.5pt}%
  \noindent\fcolorbox{adaptpinkframe}{adaptpinkbg}{%
    \parbox{\dimexpr\linewidth-2\fboxsep-2\fboxrule\relax}{%
      \small\textbf{Comment.}\enspace #1}}%
  \endgroup
  \par\smallskip
}

\graphicspath{{figures/}{pdf_figure/}}

\title{Shared Actors Need Not Share Critics:\\
Effects of Value Mismatch in\\
Parallel Reinforcement Learning}

\author{Zhenya Liu$^{1*}$, Yang Meng$^{1}$, Zhuokai Zhao$^{1}$, Xuefeng Liu$^{2\dagger}$ \& Yuxin Chen$^{1\dagger*}$\\[2pt]
$^{1}$University of Chicago\quad $^{2}$University of Florida\\[6pt]
{\normalsize
 \faGlobe\,\href{https://liu-zhenya.github.io/shared-actors-need-not-share-critics/}{\textbf{Project page}}\qquad
 \faGithub\,\href{https://github.com/Liu-Zhenya/share-actor-need-not-share-critic}{\textbf{Code}}}}

\iclrfinalcopy  
\begin{document}
\raggedbottom
\maketitle
\lhead{Preprint}
{\makeatletter\renewcommand{\thefootnote}{}%
\long\def\@makefntext#1{\noindent#1}%


{}\footnotetext{Preprint.
$^{\dagger}$Equal supervision. $^{*}$Corresponding authors
}

\makeatother\addtocounter{footnote}{-1}}
\begin{abstract}
When a single policy is trained in parallel across multiple environments of the same task,
such as procedurally generated levels, randomized dynamics, or curricula, implementations
commonly use one critic across all sampled environments.
Yet different environments can assign different expected returns to the same input visible
to the critic. A critic without environment information must then reconcile distinct value targets,
systematically shifting the sampled advantages within individual environments.
Using illustrative bandit models with multiple environments and a common optimal arm, we characterize how
this value mismatch redistributes sampled policy updates, reinforcing unhelpful actions while
attenuating or even reversing useful ones. The oracle processes using no baseline, the shared
value, or the value specific to the sampled environment have the same mean logit update at a fixed policy
and converge to the same optimal policy, yet their realized learning paths can differ sharply.
The analysis motivates a minimal intervention: give only a logged environment index to the critic so that it can separate the value targets. 
Controlled CartPole and MuJoCo experiments expose the predicted shifted values, advantages, and performance gaps. 
In the more complex BipedalWalker and Procgen settings, the same intervention yields more stable learning and higher returns. 
Across all $16$ Procgen games, the multihead conditional critic improves aggregate normalized return on $600$ unseen levels per game by $40.8\%$. 
In conclusion, the theory identifies value mismatch as a direct mechanism through which critic sharing can degrade stochastic learning dynamics, not captured by scalar estimator variance alone, and the experiments show that conditioning on an index is broadly effective in parallel reinforcement learning.

\end{abstract}

\section{Introduction}
\label{sec:intro}

Policy gradient methods optimize the expected return
$J(\theta):=\E_{\tau\sim\pi_\theta}[G(\tau)]$ by updating a stochastic policy from
sampled trajectories \citep{williams1992reinforce,suttonbarto2018}. A sampled update is
proportional to $(G_t-B_t)\nabla_\theta\log\pi_\theta(a_t\mid s_t)$, where $G_t$ is the
sampled return and $B_t$ is an action-independent baseline. Subtracting such a baseline
changes the realized update but leaves its expectation unchanged at a fixed policy
\citep{weaver2001optimal,greensmith2004variance}; a learned value function is the
standard choice \citep{schulman2016gae,schulman2017ppo}. But training is a closed
loop: the policy that generates a sample is also changed by that sample. Preserving the expected gradient
does not preserve the law of the online learning process. Stochastic
softmax policy gradient can therefore follow very different learning paths despite the
same local expected direction \citep{mei2021stochasticity}. In particular, baseline choices
alter the signs and aggressiveness of sampled updates, not merely their variance
\citep{chung2021beyond,mei2022baselines}. 

We study \emph{parallel learning across multiple environments within one task}. Procedural
generalization \citep{cobbe2020procgen}, dynamics randomization
\citep{peng2018dynamics}, and level curricula such as PLR, ACCEL, and PATH
\citep{jiang2021plr,parkerholder2022accel,liu2026active} repeatedly sample environment
variants while learning one policy that must work across them. A shared actor is therefore
the objective. Standard Procgen and level replay implementations also use one value
function across sampled levels
\citep{cobbe2020procgen,jiang2021plr,raileanu2021idaac}.

The key counterintuitive point is that \emph{the same task does not imply the same state value}.
Let $S$ denote the input or representation visible to the critic. For a fixed policy $\pi$,
define the value in environment $z$ as
$V_z^\pi(s):=\E_\pi[G\mid S=s,Z=z]$.
A value function shared across environments but without access to $Z$ is associated with
the marginalized value
\[
\bar V^\pi(s):=\E_\pi[G\mid S=s]
=\sum_z q_z^\pi(s)V_z^\pi(s),
\qquad
q_z^\pi(s):=\Prob_\pi(Z=z\mid S=s).
\]
Different dynamics or horizons can therefore produce different continuation values even at
comparable reward scales.
We suppress the dependence of $q_z^\pi(s)$ on the rollout distribution when it is clear.
The resulting value mismatch in environment $z$ is
$e_z^\pi(s):=V_z^\pi(s)-\bar V^\pi(s)$. Thus a shared value can be correct for the
information available to the critic and remain unbiased as a baseline, yet systematically
miscenter samples within individual environments.
Figure~\ref{fig:simplex} previews the resulting learning dynamics in the smallest setting
that retains this mismatch. The three processes use the same shared actor and environment
distribution; they differ only in how the sampled return is centered.

\begin{center}
\begin{minipage}{\linewidth}
\centering
\includegraphics[width=\linewidth]{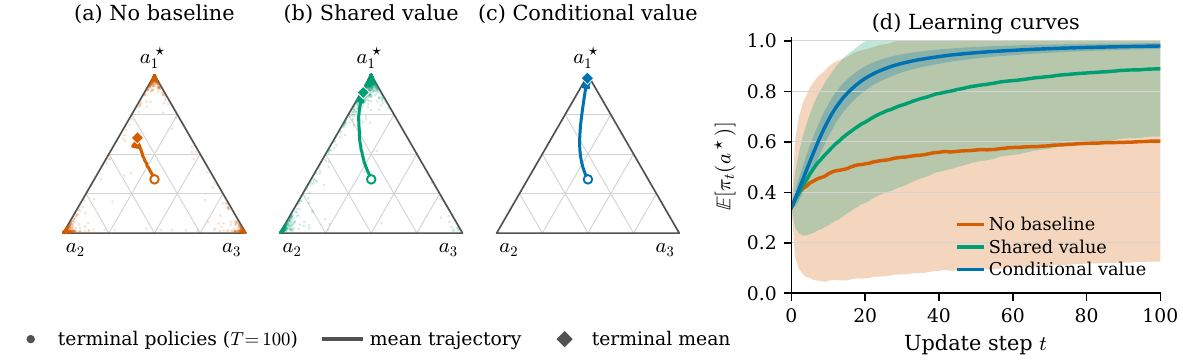}
\captionof{figure}{\textbf{Same destination, different sampled paths.}
A shared softmax actor is trained across two three-arm bandits with a common optimal arm.
The panels compare raw returns (no baseline), the shared value $\bar V^{\pi_t}$, and the
sampled environment's value $V_{Z_t}^{\pi_t}$. (a--c)~Policies after $T=100$ updates and
their mean trajectories. (d)~Mean optimal-arm probability across $4{,}000$ paired runs,
with one standard deviation. Rewards are $(1,0.9,0.8)$ and $(0.4,0.2,0)$,
$q_E=q_H=1/2$, $\eta=2$, and $\pi_0=(0.34,0.33,0.33)$.
Proposition~\ref{prop:convergence} later shows that all three processes converge to the
same optimal corner; the figure exposes their separation after only $100$ updates.}
\label{fig:simplex}
\end{minipage}
\end{center}

Conditioning only the critic has close architectural and theoretical precedents. AACC supplies
simulator factors only to the critic to support adaptation under changing dynamics
\citep{yue2024aacc}, while PAMDP conditions a dual critic on profiles for persona alignment
\citep{yang2026pamdp}. Prior theory establishes conditional value identities, unbiased
expected gradients, and aliasing benefits for privileged critics
\citep{baisero2022unbiased,li2024dcrl,lambrechts2025theoretical,ebi2026iaac}.
Our question concerns a different consequence: when environment-specific values differ,
how does their shared marginal redistribute realized updates along the closed-loop learning
path?
Our contributions are twofold. First, we explain value mismatch through the resulting
learning dynamics. In illustrative bandits (Figure~\ref{fig:simplex}), the three baselines have the
same expected update at a fixed policy and the same asymptotic destination but sharply
different sampled paths. We
characterize when shared value estimation reinforces suboptimal samples or attenuates and
reverses optimal ones. Second, this account motivates a deliberately simple intervention:
provide only an arbitrary logged environment index to the critic, allowing it to represent
a distinct value for each environment without changing the actor. For instance, controlled CartPole and
MuJoCo experiments exhibit value and advantage signatures consistent with the mechanism, and the conditioned
critic consistently outperforms the shared one. In the more complex
BipedalWalker and Procgen settings, conditioning also yields substantial empirical gains: the two designs of
conditioned critics raise the final mean return on unseen BipedalWalker terrains from $90.7$
(shared critic) to $155.6$ and $190.4$, and improve the aggregate normalized Procgen return
on unseen levels over the shared critic by $21.2\%$ and $40.8\%$, respectively. These results show that correcting value mismatch can be useful
across several forms of parallel environment training.

\section{Related work}
\label{sec:related}

\paragraph{What do baselines do?}
The classical account is variance reduction of the policy gradient estimator~\citep{williams1992reinforce,weaver2001optimal,greensmith2004variance}, refined by
control variates that depend on the action and Stein control variates~\citep{gu2017qprop,liu2018steincv,
wu2018variancereduction}. On common benchmarks, \citet{tucker2018mirage} found that the
learned baselines that depend on the action did not reduce variance beyond a
baseline conditioned only on state, and traced previously reported gains to implementation differences.
A second line studies the coupled process of sampling and updating: committal behavior~\citep{chung2021beyond}, update aggressiveness~\citep{mei2021stochasticity,
mei2022baselines}, softmax policy gradient convergence~\citep{mei2020global,li2021softmax,
agarwal2021policygradient,mei2023succeeds}, and REINFORCE convergence at any fixed
learning rate~\citep{robertson2025reinforce}. Our analysis builds directly on this view.
\citet{chung2021beyond} show in a single environment that baseline placement, rather than
variance alone, controls sampled signs and committal behavior. 
We show how sharing a value
estimate creates a structured placement error across environments: one population value
can be too low in environments with high values and too high in those with low values.
This reintroduces rival reinforcement despite using a value baseline and also suppresses
useful samples.

\paragraph{Multitask scaling and environment variation.}
PopArt normalizes heterogeneous value targets while preserving their unnormalized
predictions \citep{vanhasselt2016popart}. Multitask PopArt combines a shared policy with
value outputs indexed by task \citep{hessel2019popartmultitask}, making it a direct
architectural precedent for sharing an actor without fully sharing its critic. When a
critic without task information is shared across tasks with widely different return
scales, it must pool widely separated value targets, making scale heterogeneity an
especially visible source of value mismatch. We isolate the subtler case in which
environment variation creates different continuation values at comparable reward scales.
Target normalization controls target magnitude, whereas critic conditioning separates
environment-specific values; the two interventions address complementary aspects of the
interference.

\paragraph{Conditional and asymmetric critics.}
Conditioned values and information available only to the critic are established designs
\citep{schaul2015uvfa,pinto2018asymmetric,hu2024privileged}. Under partial observability,
\citet{baisero2022unbiased} establish a conditional value identity and an unbiased
asymmetric policy gradient; DCRL combines a critic using only history with one using both history and state to study a
variance tradeoff \citep{li2024dcrl}; and IAAC treats general privileged signals and studies
expected gradient validity and informativeness \citep{ebi2026iaac}. In complementary
theory for finite training horizons, privileged critics remove agent-state aliasing terms from linear
actor--critic bounds \citep{lambrechts2025theoretical}. These works explain validity, variance,
or aliasing benefits. We instead study how environment-dependent value offsets redistribute
realized updates even when the expected direction at a fixed policy is unchanged.

AACC is a close precedent for conditioning the critic on environment information: it learns an encoding of continuous
simulator factors for the critic and studies adaptation under changing dynamics
\citep{yue2024aacc}. Its formulation also relates values conditioned on simulator factors to their
marginal over observations alone. PAMDP instead uses a dual critic conditioned on profiles for persona
alignment \citep{yang2026pamdp}. We use an arbitrary categorical environment index, without
physical parameters, ordering, or profile semantics, as a controlled intervention on value
estimation. Our contribution is a complementary account of the learning dynamics: environment
information changes how the same mean policy gradient update is distributed across sampled
branches and thereby changes the realized path. When one critic prediction represents
environments with different futures, the same bar can promote suboptimal samples in some
environments and attenuate or reverse optimal samples in others. This is a structured
instance of perceptual aliasing \citep{chrisman1992aliasing,singh1995softaggregation}; we
characterize its consequences for sampled paths and test the index intervention across parallel
RL benchmarks.

\section{Problem setting: one policy, many environments}
\label{sec:setting}

\paragraph{An explanatory bandit model with multiple environments.}

We isolate the effect of critic sharing in the smallest model that retains a value
that depends on the environment.  Let $\mathcal Z$ be a finite collection of environments and
$\mathcal A$ a finite set of $K\ge2$ arms.  At round $t$, an environment index
$Z_t\sim q$ is drawn independently, where $q_z>0$ and $\sum_zq_z=1$.  The actor is a
single softmax actor
$\pi_t(a)\propto\exp\theta_t(a)$ that samples $a_t\sim\pi_t$ without observing
$Z_t$.  Environment $z$ assigns a deterministic scalar reward $r_z(a)$ to
arm $a$; because both sets are finite, the reward table is uniformly bounded.

We assume that the environments share one strict optimal arm: $r_z(a^*)>r_z(i)$ for
every $z\in\mathcal Z$ and $i\ne a^*$. Because $a^*$ is optimal in every environment,
any sampled update that decreases $\pi(a^*)$ cannot be attributed to conflicting
objectives across environments; it reflects how that sampled branch is centered and
updated.

For a policy $\pi$, define the environment-specific value
$V_z^\pi:=\E_{a\sim\pi}[r_z(a)]=\sum_a\pi(a)r_z(a)$ and its shared marginal
$\bar V^\pi:=\E_{Z\sim q}[V_Z^\pi]=\sum_zq_zV_z^\pi$.
At round $t$, we compare the three oracle baselines $B_t^0=0$,
$B_t^{\mathrm{shared}}=\bar V^{\pi_t}$, and
$B_t^{\mathrm{cond}}=V_{Z_t}^{\pi_t}$.
These oracle quantities isolate the effect of centering from critic fitting; the experiments
study the same intervention with learned critics.

\begin{algorithm}[h]
\caption{Softmax Policy Gradient across Environments}
\label{alg:eispg}
\begin{algorithmic}[1]
\Require $q$, rewards $\{r_z\}$, step size $\eta$, baseline $B\in\{B^0,B^{\mathrm{shared}},B^{\mathrm{cond}}\}$
\For{$t=0,1,2,\dots$}
  \State draw $Z_t\sim q$ and $a_t\sim\pi_{\theta_t}$
  \State observe $G_t=r_{Z_t}(a_t)$ and evaluate the baseline $B_t$
  \State $\theta_{t+1}\gets\theta_t+\eta\,(G_t-B_t)(\mathbf e_{a_t}-\pi_{\theta_t})$
\EndFor
\end{algorithmic}
\end{algorithm}

\noindent Algorithm~\ref{alg:eispg} is the complete stochastic process analyzed in the theory.
For the full softmax logit parameterization,
$\nabla_\theta\log\pi_\theta(a_t)=\mathbf e_{a_t}-\pi_\theta$,
so Line~4 is one sampled REINFORCE update with an action-independent baseline.

The bandit is an explanatory abstraction. At a matched input visible to the critic in an MDP,
$r_z(a)$ represents the return for action $a$ in environment $z$. The
abstraction isolates value mismatch from state visitation, critic estimation error, and
function approximation.

\paragraph{Why convergence first.}
We first establish a common asymptotic result to isolate path effects over a finite horizon
and to prove the two facts required below: entry into a near-optimal region and infinite exploration.

\begin{proposition}[A common destination and a rate for time averages]
\label{prop:convergence}
Consider Algorithm~\ref{alg:eispg} with finite initial logits, the common strict optimal arm
defined above, and any fixed learning rate $\eta\in(0,\infty)$. For each
$B\in\{B^0,B^{\mathrm{shared}},B^{\mathrm{cond}}\}$,
\begin{equation}
\label{eq:prop1-convergence}
\Prob\!\left(
\lim_{t\to\infty}\pi_t^B(a^*)=1
\right)=1.
\end{equation}
Moreover, given $\bar r(a):=\sum_z q_z r_z(a)$, for each such $B$,
almost surely there exist finite random constants $C_B$ and
$T_{0,B}$ such that, for every integer
$T\ge \max\{T_{0,B},2\}$,
\begin{equation}
\label{eq:prop1-rate}
\frac{1}{T}\sum_{t=0}^{T-1}
\left[
\bar r(a^*)-\sum_a\pi_t^B(a)\bar r(a)
\right]
\le
C_B\frac{\log T}{T}.
\end{equation}
\end{proposition}
Proposition~\ref{prop:convergence} serves two roles. First, from any initialization with
finite logits and for any fixed finite $\eta>0$, all three processes reach the same optimal
policy, and each has an $O(\log T/T)$ upper bound on time-averaged suboptimality. Their separation
under finite training budgets must therefore come from how they travel, not from their final destination.
Second, the proposition and its proof establish entry into a near-optimal region and infinite
exploration for Propositions~\ref{prop:single-ratchet} and~\ref{prop:pool-ratchet};
under their stated conditions, the regimes of eventual ratcheting and recurring drawdowns are
reached almost surely rather than merely characterized conditionally. These quantifiers
also show that, whenever the mismatch condition holds, the resulting pathwise effect is not an
artifact of a favorable initialization or of choosing a small step size. Learning rate can still
control the severity over a finite horizon: in the fixed Appendix instance, larger $\eta$ amplifies
the updates on reversed branches without implying a general monotone ordering of return across learning
rates (Figure~\ref{fig:learning-rate-sweep}).

At a fixed policy, averaging over the sampled arm within each environment cancels every
action-independent baseline, and averaging over the environment leaves only the pooled reward vector
$\bar r$. Thus the three processes share the same expected update when evaluated at
the same policy. Their sampled updates nevertheless place them at different policies,
so later rounds evaluate that common mean direction at different points. The common
arm $a^*$ is the unique maximizer of $\bar r$. Appendix~\ref{app:prop1} derives the
exact identity for the mean update and gives a self-contained proof of
Proposition~\ref{prop:convergence}, while explaining its relation to
\citet{robertson2025reinforce}.
The rate is a statement about a time average within each process: it neither orders the last iterates
at a finite $T$ nor forces the random entrance times and constants to agree across
baselines. The realized stochastic processes can therefore differ sharply at any finite
time.
Figure~\ref{fig:simplex}
runs the exact process of Algorithm~\ref{alg:eispg} on an instance with three arms and two environments: after
$T=100$ updates, the three schemes occupy sharply different regions of the simplex even
though all converge to the same corner asymptotically.  The rest of the theory explains this separation
through realized update branches.

\section{Baselines change the online update dynamics}
\label{sec:mechanics}

Proposition~\ref{prop:convergence} establishes the common destination and bounds an
optimality gap averaged over time; it does not determine how Algorithm~\ref{alg:eispg} travels.
At a fixed policy, every action-independent baseline yields the identical expected update,
but the
algorithm never takes that expected step: each round draws a single action and applies
the realized update of that branch alone. Baselines that agree in expectation can
therefore still differ in the magnitude and even the sign of individual realized updates.
Because the policy determines which actions are sampled and those samples in turn update
the policy, these branch differences accumulate into distinct stochastic processes.

Variance gives the classical aggregate account of stochastic optimization.  For
unbiased estimators with the same mean, smaller variance tightens the smooth SGD lower
bound on expected improvement after one step and, with standard continuity and boundedness
conditions, yields sharper convergence guarantees \citep{bottou2018optimization}.
Yet this account is weak as an explanation of the learning process in parallel RL: even in the
elementary bandit, the exact minimum variance baseline can produce a slower and less
stable approach over finitely many steps than the ordinary conditional value baseline
(Appendix~\ref{app:variance-view}).  The scalar variance does not record which sampled
environment and action branches are reinforced.  We therefore need a more interpretable
theory that tracks the coupling between sampling by the policy and updates from individual
samples. This coupling determines the realized optimization path and directly affects
performance under the finite training budgets used in reinforcement learning.

\subsection{One environment: an eventual ratchet}
\label{sec:single-ratchet}

A value baseline is a moving bar: only an arm whose reward exceeds the current value is
reinforced.  In two arms, both possible samples therefore move the policy toward the
better arm; without a baseline, sampling a rival with positive reward reinforces the mistake.
The same distinction eventually holds for any finite number of arms. In the case of one environment, the shared and conditional baselines coincide, and we write $\pi^V$ and $\pi^0$ for the policies updated with the value baseline and without a baseline, respectively.

\begin{proposition}[Eventual ratchet versus persistent drawdowns]
\label{prop:single-ratchet}
Consider a deterministic finite bandit with finite initial logits, a unique optimal arm
$a^*$, and any fixed $\eta>0$.  Under the oracle value baseline,
\begin{equation}
\Prob\!\left(\exists T<\infty:\
\pi^V_{t+1}(a^*)>\pi^{V}_t(a^*)\ \text{for every }t\ge T\right)=1.
\label{eq:single-ratchet}
\end{equation}
If at least one rival has positive reward, the process without a baseline instead satisfies
\begin{equation}
\Prob\!\left(\pi^{{0}}_{t+1}(a^*)<\pi^{{0}}_t(a^*)\ \text{infinitely often}\right)=1.
\label{eq:no-eventual-ratchet}
\end{equation}
\end{proposition}

Both chains converge to $a^*$, but only value centering eventually turns every sample
into progress.  Appendix~\ref{app:single-ratchet} gives the threshold, branch algebra,
and argument based on infinite exploration behind these two probability statements.

\subsection{Multiple environments: the shared offset}
\label{sec:mismatch}

We now extend to multiple environments. Recall that $S$ denotes the input or representation
visible to the critic. Environment differences hidden from this representation can change
continuation values while forcing an unconditioned critic to assign them one marginalized
value $\bar V^\pi(s)=\sum_z q_z^\pi(s)V_z^\pi(s)$, where
$q_z^\pi(s)=\Prob_\pi(Z=z\mid S=s)$.
In the bandit model, this conditional mixture reduces to the fixed environment weight
$q_z$. Define the value mismatch by
$e_z^\pi(s):=V_z^\pi(s)-\bar V^\pi(s)$. For one sampled return $G$, the conditional and
shared advantages satisfy
\[
A^{\mathrm{cond}}:=G-V_z^\pi(s),
\qquad
A^{\mathrm{shared}}:=G-\bar V^\pi(s)
=A^{\mathrm{cond}}+e_z^\pi(s).
\]
Thus, for the same sampled environment and action, value mismatch shifts the scalar
multiplying the score vector by $e_z^\pi(s)$; Appendix~\ref{app:mismatchproofs} gives
the exact identity for the logit update. The offset is agnostic to its source: heterogeneous reward scales and
different dynamics can both produce value mismatch.
Three facts make this offset a genuine problem rather than a transient.
\emph{First, it need not fade as the policy improves}: for a fixed limiting mixture,
the offsets converge to constants that depend on the environment, with at least one nonzero
whenever the limiting optimal values are not all equal; an offset can even grow along
training, so better training does not repair it.
\emph{Second, it can systematically change update signs}: environments above the average
have $e_z^\pi(s)>0$, so
suboptimal draws can clear the bar and be reinforced --- the feedback branch of
Section~\ref{sec:single-ratchet} returns; environments below the average have
$e_z^\pi(s)<0$, so even the
optimal draw can fall below the bar and be suppressed.
\emph{Third, it makes the shared baseline differently aggressive across environments}: the
shared critic removes the value level averaged across environments, but within each
environment it misplaces the bar by $e_z^\pi(s)$. The conditional critic $V_z^\pi(s)$
removes this offset. We illustrate these effects below.

\par\smallskip
\noindent\begin{minipage}{\linewidth}
\centering
\includegraphics[width=\linewidth]{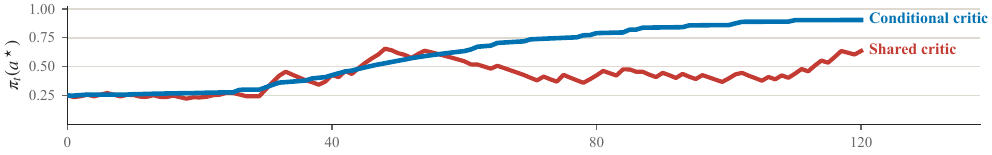}
\vspace{-9pt}
\captionof{figure}{A sample trajectory of Algorithm~\ref{alg:eispg} under the shared
and the conditional value baseline ($q_E=q_H=\tfrac12$, $r_E=(1,0.8,0.6)$,
$r_H=(0.4,0.2,0)$, $\eta=0.5$): the conditional path increases at every update
after an early transient, while the shared path keeps stepping backward.}
\label{fig:theory-path-intuition}
\end{minipage}
\par\smallskip

\begin{proposition}[Conditional ratchet versus persistent shared drawdowns]
\label{prop:pool-ratchet}
Consider Algorithm~\ref{alg:eispg} with finite initial logits and any fixed $\eta>0$.
Under the oracle conditional baseline,
\begin{equation}
\Prob\!\left(\exists T<\infty:\
\pi^{B^{\mathrm{cond}}}_{t+1}(a^*)>\pi^{B^{\mathrm{cond}}}_t(a^*)\ \text{for every }t\ge T\right)=1.
\label{eq:conditional-pool-ratchet}
\end{equation}
If the reward $r_z(a^*)$ of the optimal arm is not identical across environments, the shared
baseline process satisfies
\begin{equation}
\Prob\!\left(\pi^{B^{\mathrm{shared}}}_{t+1}(a^*)<\pi^{B^{\mathrm{shared}}}_t(a^*)\ \text{infinitely often}\right)=1.
\label{eq:shared-persistent-drawdown}
\end{equation}
\end{proposition}

The proposition concerns realized updates, not only expected drift. Conditional centering
eventually makes every sample, regardless of its environment and arm, move the shared actor toward the optimum. Shared
centering converges to the same policy but continues to step backward whenever an
environment below the average supplies a sample of the optimal arm sufficiently late. Those events
have asymptotic frequency equal to the total sampling mass of environments below the
average, while their magnitudes vanish near the limit.
Mismatch severity controls when this regime begins and how strongly it acts: a larger
negative offset lets optimal samples from hard environments fall below the shared bar while the
policy is still farther from its limit, and produces larger reversals thereafter; a
larger positive offset lets more rivals sampled in easy environments clear the bar.
Appendix~\ref{app:mismatchproofs} gives the exact thresholds, magnitudes, and frequency
statement. More generally, the effect over a finite training horizon depends jointly on the current policy,
mismatch magnitude, and learning rate: larger mismatch can trigger reversals farther from
optimality, while a larger $\eta$ amplifies each reversed update and its effect on subsequent
sampling. Through this sampling and update feedback, earlier reversals and their larger
magnitudes can compound, producing a less favorable trajectory over a finite training horizon
even though the asymptotic destination is unchanged
(Figure~\ref{fig:theory-path-intuition}).

\paragraph{Why the mechanism harms real RL.}
In Algorithm~\ref{alg:eispg}, each sampled score contribution reinforces its arm exactly when the
centered return is positive --- committal behavior in the sense of
\citet{chung2021beyond}. Minibatch PPO aggregates and transforms many such contributions,
so this is a mechanism at the signal level rather than an exact claim about the net optimizer
step. In real RL the bar is learned. The oracle shared value above is defined by
marginalizing over the unobserved environment identity. Lemma~\ref{lem:projection}
separately connects this quantity to value fitting: under regression with squared error, a critic
that observes $s$ but not $z$ has $\bar V^\pi(s)$ as its population target. Thus the
shift across environments is present in the regression target itself rather than arising only
from noise due to finite samples. Bootstrapping, approximation, and PPO transformations affect
how closely a learned critic realizes that target. As mismatch grows, both sides worsen: mistakes in easy environments receive a
larger positive lift, while useful updates in hard environments are attenuated earlier and
more strongly.  This attenuation matters even before an advantage changes sign: once
the policy is sufficiently close to the optimal corner, weakening its frequent positive
updates can cost more probability than the stronger negative reinforcement of rare rivals
recovers (Lemma~\ref{lem:hard-attenuation} in Appendix~\ref{app:mismatchproofs}).
In practice, easy environments can supply many plausible but inferior
trajectories while useful trajectories from hard environments are rare; shared value estimation can reinforce
the former and attenuate or reverse the latter.  Conditioning removes this value offset,
though not genuine disagreement about the best action. Modern PPO commonly uses
Generalized Advantage Estimation (GAE), which interpolates between one-step bootstrapping
and Monte Carlo returns \citep{schulman2016gae,schulman2017ppo}. On the same rollout, define
$D_t^\lambda:=\widehat A_t^{\mathrm{shared},\lambda}
-\widehat A_t^{\mathrm{cond},\lambda}$.
Appendix~\ref{app:mismatchproofs} shows that $D_t^\lambda$ is a temporal filter of the
sequence of value mismatches. Thus GAE propagates and mixes mismatch across a rollout rather
than removing it.

Figure~\ref{fig:hero} shows the mechanism in a minimal example with learned critics:
two environments share one state (Figure~\ref{fig:hero}A). The shared critic mean
approaches the average across environments (Figure~\ref{fig:hero}B), shifting mean
advantages above zero in the easy environment and below zero in the hard environment
(Figure~\ref{fig:hero}C,D). The resulting policy trajectory appears in
Figure~\ref{fig:hero}E, while conditioning delivers advantages with the correct signs.
\begin{figure}[h]
\centering
\includegraphics[width=\linewidth]{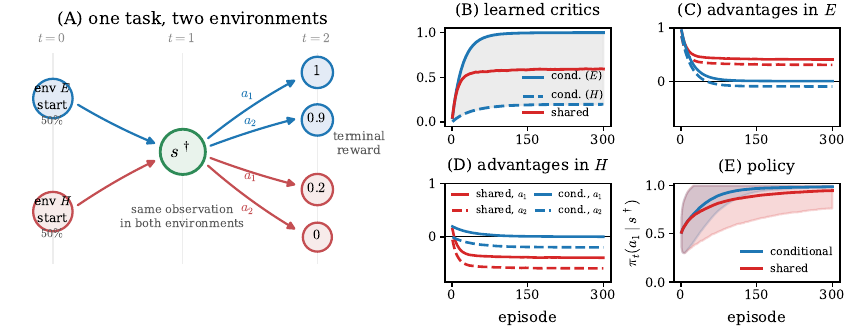}
\vspace{-10pt}
\caption{A minimal tabular example with \emph{learned} critics ($2{,}000$ seeds, $\eta=1$,
critic learning rate $\beta=0.1$, common random numbers; logits and both critics initialized at
zero, so the initial policy is uniform).
(A)~Two environments E and H, drawn $50\%$ each, share exactly one state $s^\dagger$;
the same actions lead to endpoints worth $(1,\,0.9)$ in $E$ and $(0.2,\,0)$ in $H$, so
$a_1$ is optimal in both.
(B)~The conditional critic means track each environment; the shared critic mean approaches
their average ($\to0.6$).
(C,D)~Mean advantages in $E$ and $H$: sharing promotes $E$'s wrong action ($+0.3$)
and suppresses $H$'s correct action ($-0.4$), while conditioning gives the correct signs.
(E)~$\pi_t(a_1\mid s^\dagger)$, mean$\,\pm\,$1 s.d.: the conditional baseline dominates
the shared one throughout training.}
\label{fig:hero}
\end{figure}

\subsection{Method: condition only the critic}
\label{sec:method}

The analysis prescribes a minimal intervention on the critic: give only a logged
environment index $z$ to the critic and leave the actor and sampling process unchanged.
The index is arbitrary and categorical; it does not explicitly provide geometry, ordering,
difficulty, physical parameters, or behavioral semantics. It only identifies which
value target for that environment the critic should fit. The intervention therefore
directly tests the optimization effect of correcting value mismatch.
We instantiate it with the
two basic conditioning designs available for a fixed set of environments:
\textbf{FiLM} \citep{perez2018film}, an affine scale and shift for each environment on shared critic
features, and a \textbf{multihead} critic, a shared encoder with one value head per
environment \citep{hessel2019popartmultitask}. Both are standard components: FiLM introduces two
additional vectors per environment, and a multihead critic introduces one environment-specific linear readout; the critic's forward pass is otherwise unchanged, and no additional
rollouts, updates, or inference machinery are required, since $z_t$ is a logged index
and requires no estimation. Both retain shared structure while allowing the critic to
represent the offset $e_z^\pi(s)$ identified above as the harmful object, and the actor never
receives $z$ during training. At deployment the critic and logged index are discarded,
leaving the same shared actor.

In the oracle model, conditioning removes the offset exactly. Learned critics only
approximate these values for individual environments; the experiments test whether the
same intervention improves the real RL optimization process.

\section{Experiments}
\label{sec:experiments}

The experiments are organized around three roles. CartPole
\citep{barto1983neuronlike} provides a controlled, end-to-end identification of the
value mismatch mechanism, from conflicting value targets to shifted advantages and the
resulting learning behavior.
MuJoCo tests whether the predicted advantage structure persists with continuous states,
function approximation, and hidden dynamics variation. BipedalWalker
\citep{brockman2016gym} and Procgen \citep{cobbe2020procgen} then evaluate the practical
value and scalability of critic conditioning across $100$--$200$ procedural environments.
All experiments use PPO with Generalized Advantage Estimation (GAE)
\citep{schulman2016gae,schulman2017ppo}. In every main comparison, only the critic is
modified: conditioned variants receive the logged environment index, while the actor,
sampling protocol, and PPO pipeline remain shared. Full architectures, protocols, and
hyperparameters are in
Appendix~\ref{app:expdetails}.

\subsection{CartPole}
\label{sec:cartpole}

CartPole uses two logged levels with the same observation space and reward function. In the
heterogeneous setting their gravities are $g\in\{10,50\}$; an identical control assigns
$g=10$ to both level identities. Figure~\ref{fig:cartpole}(a--c) shows why the
heterogeneous pair creates value mismatch: the same state and action lead to different
futures. Correspondingly, the learned multihead values separate by gravity, whereas the
shared value lies between them (Figure~\ref{fig:cartpole}d).

\begin{figure}[h]
\centering
\hspace*{-8mm}%
\includegraphics[width=\linewidth]{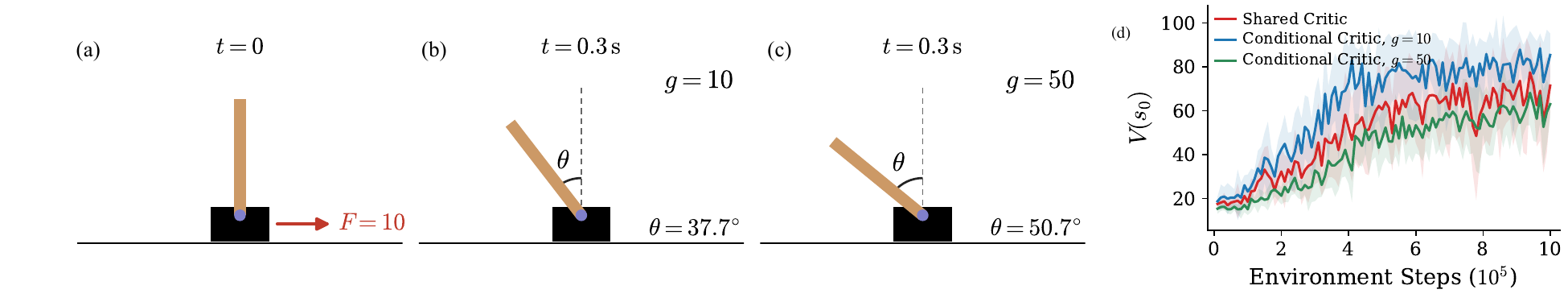}
\caption{CartPole with two gravities. (a--c)~From the same initial state and
the same push ($F=10$), $0.3$\,s of dynamics tilts the pole to $37.7^\circ$ under
$g=10$ and $50.7^\circ$ under $g=50$: one observation, two futures.
(d)~Learned value of the initial state over training: the multihead critic means separate
by gravity (blue: $g{=}10$ above, green: $g{=}50$ below), while the shared critic mean (red)
settles between them --- a signature consistent with the value mismatch mechanism. Curves
in (d) show mean$\,\pm\,$1 s.d.\ over $20$ seeds.}
\label{fig:cartpole}
\end{figure}

To identify what the critic must learn, we compare four variants: a shared critic; a
multihead critic routed by the true level identity; the same multihead critic with
every sample routed to head $0$; and a shared critic augmented by one learned scalar bias
for each level. All four begin with identical value predictions. The identical gravity
control tests whether conditioning helps in the absence of mismatch, while the
heterogeneous pair orders the variants by the environment information and correction they
can represent.

\begin{figure}[t]
\centering
\includegraphics[width=\linewidth]{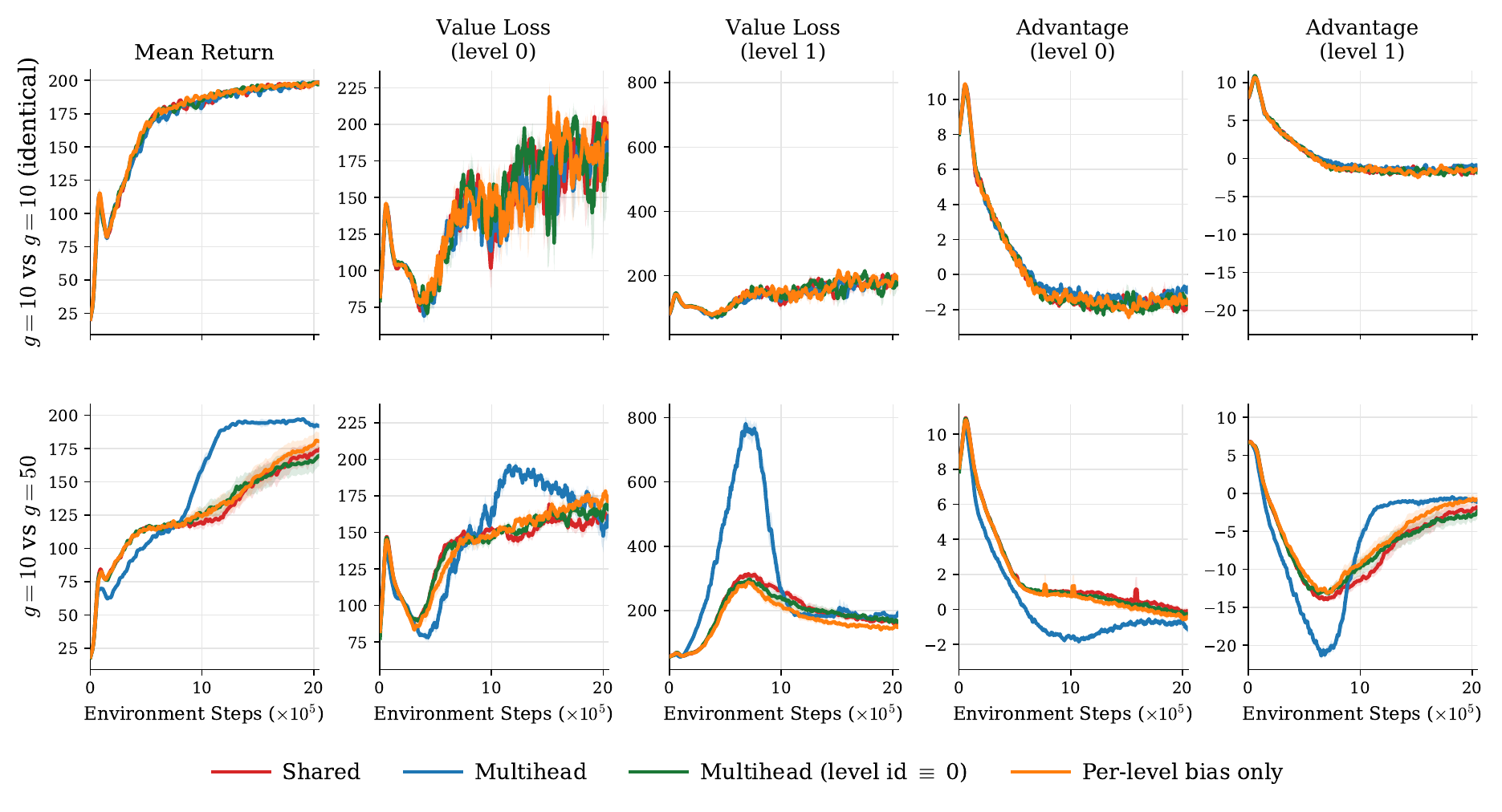}
\vspace{-10pt}
\caption{Controlled identification of the CartPole mechanism, mean$\,\pm\,$1 s.e.\ across
$20$ seeds. The top row uses two labeled levels with identical dynamics ($g=10$); the
bottom row uses $g\in\{10,50\}$. Columns report mean return, value loss for levels $0$ and
$1$, and the mean sampled raw GAE advantage for levels $0$ and $1$. With identical
dynamics, all variants behave similarly. Under heterogeneous gravity, the multihead critic
initially pays a larger fitting cost on the hard level, but its hard level advantage recenters
as the loss falls and its return subsequently approaches $200$. Routing every sample to the
same head (level id $\equiv 0$) closely tracks the shared critic, while a learned scalar bias
for each level provides only a partial correction.}
\label{fig:cartpole-mechanism}
\end{figure}

The row with identical dynamics shows no systematic separation, as expected in the absence of
value mismatch. Under heterogeneous gravity, the shared critic keeps the hard level's mean
advantage negative while its return remains lower and more variable across seeds
(Figure~\ref{fig:cartpole-mechanism}). Conditioning does not supply an oracle value: each
head must learn online from samples of its own level while the shared representation is
changing. In this compact joint actor--critic model, the resulting fitting transient is
clearly visible: the multihead loss on the hard level initially peaks near $780$, and its mean
advantage falls to roughly $-22$.

The contrast emerges once the hard head catches up. Its mean advantage recenters earlier and
the multihead return rises to a narrow band near $200$, while the shared critic and the critic
given a constant index remain lower with wider variation. Routing every sample to the same head closely
tracks the shared critic, showing that additional heads without informative level routing do
not reproduce the gain. The critic with one learned scalar bias per level recovers only part of the performance,
indicating that a substantial component of the required correction is state dependent. From
mid to late training, the shared critic shifts the mean advantage upward on level $0$ and
sharply downward on level $1$ relative to the multihead critic. The corresponding return gap
matches the predicted pattern of positive and negative value mismatches.

\subsection{MuJoCo}
\label{sec:mujoco}

For HalfCheetah, Hopper, and Walker2d \citep{todorov2012mujoco}, we vary body mass over
ten levels without providing it explicitly in the observation, while keeping episodes at a
fixed length; only the FiLM critic receives the level index. FiLM has higher mean return on all three tasks
(Figure~\ref{fig:mujoco}, right middle), with the largest gains on HalfCheetah (roughly
$2600$ vs.\ $2050$) and Walker2d ($2450$ vs.\ $2000$); Hopper shows a smaller gap in the end, but the conditional critic still dominates the shared one during training. More importantly, the advantage heatmaps show the predicted
structure: the shared critic shifts light levels positive and heavy levels negative,
whereas conditioning largely removes this ordering (Figure~\ref{fig:mujoco}, right bottom).

\begin{figure}[h]
\centering
\begin{minipage}[b]{0.35\linewidth}
\vspace{0pt}
\centering
\includegraphics[width=\linewidth]{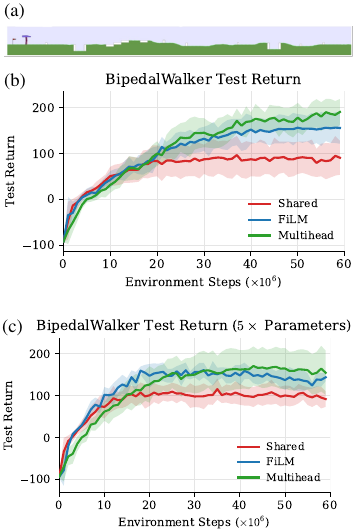}
\end{minipage}\hspace{0.02\linewidth}%
\begin{minipage}[b]{0.62\linewidth}
\vspace{0pt}
\centering
\raisebox{-2mm}{%
\includegraphics[width=\linewidth]{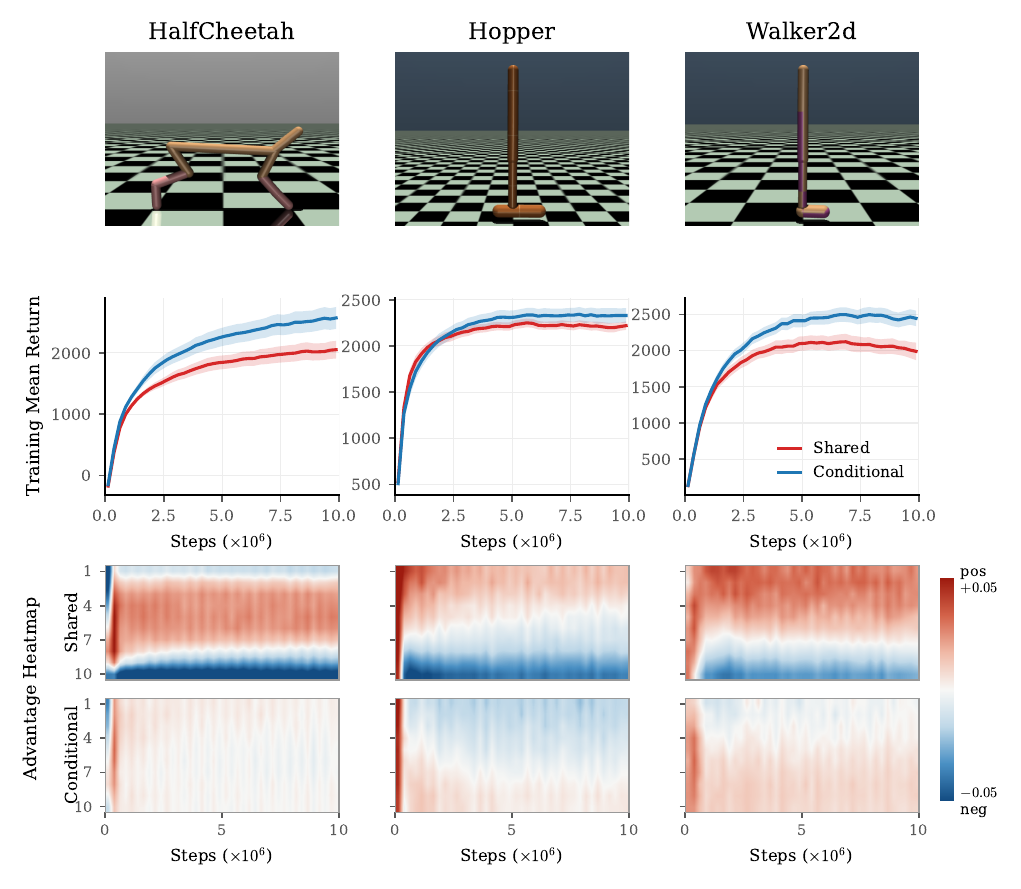}}
\end{minipage}
\caption{Continuous control benchmarks. \textbf{Left:} BipedalWalker with
$100$ pinned terrains: \textbf{(a)}~terrain sample, \textbf{(b)}~test return on $100$
unseen terrains, and \textbf{(c)}~the control with $5\times$ as many parameters
(Section~\ref{sec:ablations}); curves show mean$\,\pm\,$1 s.d.\ over $10$ seeds.
\textbf{Right:} MuJoCo with ten body mass levels not explicitly provided in the observation
($10$M steps). Columns are
HalfCheetah, Hopper, and Walker2d; the middle row shows training return,
mean$\,\pm\,$1 s.d.\ over $30$ seeds, and the bottom row shows continuous means across seeds of GAE advantages for each of $10$ levels during training
(light to heavy; red positive).}
\label{fig:continuous-control}
\label{fig:bipedal}
\label{fig:mujoco}
\end{figure}

\subsection{BipedalWalker}
\label{sec:bipedal}

BipedalWalker uses a training set of $100$ pinned terrains dominated by hard terrains: $10$
standard and $90$ Hardcore (Figure~\ref{fig:bipedal}, left a). Only the critic sees terrain
identity, and evaluation uses $100$ unseen terrains with the same $10/90$
standard/Hardcore split. The shared critic plateaus below
$100$, while FiLM reaches roughly $150$ and multihead $165$--$190$, with visibly
tighter bands (Figure~\ref{fig:bipedal}, left b). Both conditioning architectures escape the
same plateau, indicating that the effect is not specific to one architecture.

\subsection{Procgen}
\label{sec:procgen}

Each of the $16$ Procgen games is trained separately on $200$ pinned training levels for $25$M
steps; the critic may see level identity, but the actor never does. This setting asks
whether an intervention using only a level index remains useful when value mismatch is distributed
across hundreds of procedurally distinct environments rather than a small ordered family.

Evaluation on unseen levels shows that the effect extends beyond the pinned training levels.
FiLM or multihead has a higher final mean than the shared critic in all of the
$16$ games. Relative to the shared critic, their aggregate normalized returns on unseen levels
improve by $21.2\%$ and $40.8\%$, respectively (Table~\ref{tab:heldout}).

The training learning curves separately expose the optimization behavior. Improvements
appear across navigation, control, and arcade games in the illustrative subset in
Figure~\ref{fig:procgen}; the full set is in Appendix Figure~\ref{fig:procgen-full}. Complete
training curves and final returns are in Appendix Table~\ref{tab:popart}.

\begin{figure}[t]
\centering
\includegraphics[width=\linewidth]{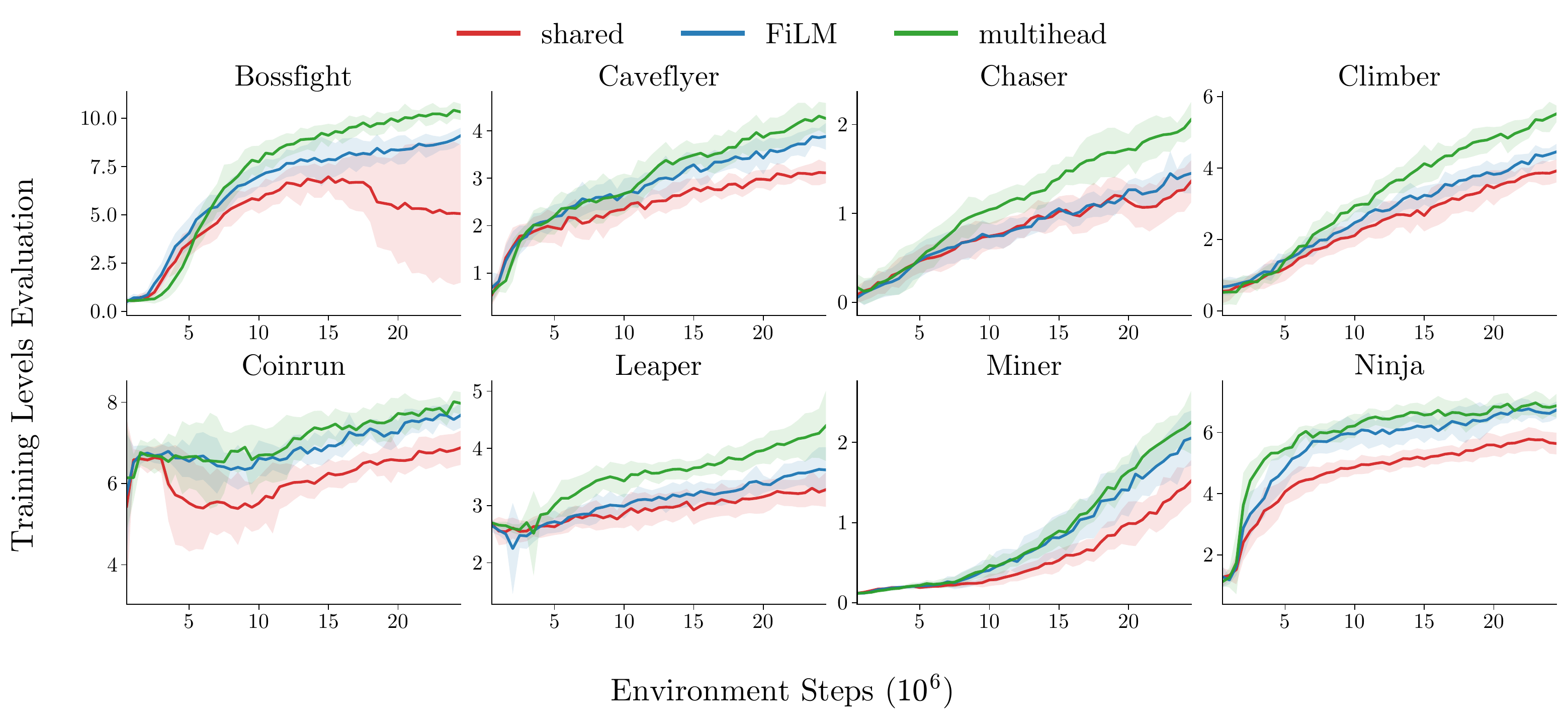}
\vspace{-10pt}
\caption{An illustrative subset of eight Procgen games: evaluation return on the $200$ pinned training levels
over $25$M steps, mean$\,\pm\,$1 s.d.\ over $10$ seeds, shared (red) vs.\ FiLM (blue)
vs.\ multihead (green). Conditioned critics train to higher returns broadly; in
Bossfight the shared critic mean collapses after roughly $15$M steps while the
conditioned critic means keep climbing.}
\label{fig:procgen}
\end{figure}

\begin{table}[h]\centering\small
\caption{Procgen, $16$ games: final evaluation return on $600$ unseen
levels ($25$M steps; mean$\,\pm\,$1 s.d.\ over $10$ seeds per method). Normalized returns per run are computed by dividing the average test return per run for each environment by the corresponding
average test return of the shared critic baseline over all runs. \textbf{Bold} $=$ best in the row.}
\label{tab:heldout}
\begin{tabular}{l rrrr}
\toprule
& \multicolumn{4}{c}{evaluation return (600 unseen levels)} \\
\cmidrule(lr){2-5}
Game & shared & FiLM & multihead & \,+\,PopArt \\
\midrule
Bigfish & 1.87\,\tiny$\pm$0.17 & 2.88\,\tiny$\pm$0.57 & 2.99\,\tiny$\pm$0.61 & \textbf{3.56\,\tiny$\pm$1.27} \\
Bossfight & 4.44\,\tiny$\pm$2.56 & 7.44\,\tiny$\pm$0.51 & \textbf{8.97\,\tiny$\pm$0.29} & 8.11\,\tiny$\pm$0.82 \\
Caveflyer & 1.08\,\tiny$\pm$0.16 & 1.90\,\tiny$\pm$0.23 & \textbf{2.65\,\tiny$\pm$0.43} & 1.87\,\tiny$\pm$0.31 \\
Chaser & 0.95\,\tiny$\pm$0.26 & 0.94\,\tiny$\pm$0.19 & \textbf{1.32\,\tiny$\pm$0.35} & 0.75\,\tiny$\pm$0.17 \\
Climber & 1.40\,\tiny$\pm$0.19 & 1.74\,\tiny$\pm$0.27 & \textbf{2.75\,\tiny$\pm$0.22} & 2.13\,\tiny$\pm$0.27 \\
Coinrun & 5.76\,\tiny$\pm$0.29 & 6.35\,\tiny$\pm$0.28 & \textbf{6.80\,\tiny$\pm$0.37} & 6.46\,\tiny$\pm$0.29 \\
Dodgeball & 0.98\,\tiny$\pm$0.05 & 0.91\,\tiny$\pm$0.14 & 1.10\,\tiny$\pm$0.20 & \textbf{1.32\,\tiny$\pm$0.18} \\
Fruitbot & 7.39\,\tiny$\pm$1.58 & \textbf{8.95\,\tiny$\pm$0.70} & 8.14\,\tiny$\pm$1.24 & 8.31\,\tiny$\pm$1.01 \\
Heist & 0.24\,\tiny$\pm$0.06 & \textbf{0.28\,\tiny$\pm$0.05} & 0.20\,\tiny$\pm$0.03 & 0.24\,\tiny$\pm$0.06 \\
Jumper & 2.23\,\tiny$\pm$0.13 & \textbf{2.36\,\tiny$\pm$0.15} & 2.32\,\tiny$\pm$0.12 & 2.35\,\tiny$\pm$0.62 \\
Leaper & 2.81\,\tiny$\pm$0.36 & 3.02\,\tiny$\pm$0.24 & \textbf{3.25\,\tiny$\pm$0.46} & 3.09\,\tiny$\pm$0.28 \\
Maze & 1.23\,\tiny$\pm$0.13 & 1.33\,\tiny$\pm$0.13 & 1.49\,\tiny$\pm$0.18 & \textbf{1.49\,\tiny$\pm$0.15} \\
Miner & 0.53\,\tiny$\pm$0.12 & 0.58\,\tiny$\pm$0.08 & \textbf{0.80\,\tiny$\pm$0.14} & 0.77\,\tiny$\pm$0.13 \\
Ninja & 3.39\,\tiny$\pm$0.21 & 3.98\,\tiny$\pm$0.15 & \textbf{4.39\,\tiny$\pm$0.29} & 4.10\,\tiny$\pm$0.41 \\
Plunder & 2.37\,\tiny$\pm$0.33 & 2.38\,\tiny$\pm$0.18 & 2.28\,\tiny$\pm$0.18 & \textbf{2.42\,\tiny$\pm$0.29} \\
Starpilot & 7.96\,\tiny$\pm$1.06 & 10.22\,\tiny$\pm$1.11 & 13.10\,\tiny$\pm$1.46 & \textbf{14.98\,\tiny$\pm$2.35} \\
\midrule
\textbf{Normalized return (\%)} & 100.0\,\tiny$\pm$19.8 & 121.2\,\tiny$\pm$27.9 & \textbf{140.8\,\tiny$\pm$47.1} & 133.3\,\tiny$\pm$41.7 \\
\bottomrule
\end{tabular}
\end{table}

\paragraph{Conditional value fitting is not a persistent bottleneck.}
The early fitting transient in the compact CartPole study does not persist systematically on
the larger benchmarks. On BipedalWalker, the initially higher multihead loss falls to a
comparable level late in training; on all three MuJoCo tasks, FiLM remains below the shared critic
through most of training; and on Procgen, both FiLM and multihead finish below the shared
critic in $11$ of $16$ games (Appendix Figures~\ref{fig:value-loss-continuous}
and~\ref{fig:value-loss-procgen}). Thus the return gains are not accompanied by a
systematic deterioration in critic fitting.

\subsection{Ablations}
\label{sec:ablations}

\paragraph{Conditioning, not value network capacity.}
Conditioning changes the parameter count, although the increase is modest in several settings
(Appendix Table~\ref{tab:paramcounts}). More decisively, on BipedalWalker we widen each critic
architecture to roughly five times its original size while holding its conditioning mechanism
fixed. The enlarged shared critic still plateaus, whereas FiLM and multihead remain near $150$
or above (Figure~\ref{fig:bipedal}c). Parameter count alone therefore does not reproduce the
conditioning gain.

\paragraph{Conditioning, not target normalization.}
A second possibility is that the gains come from correcting target scale rather than
separating values by environment. FiLM and multihead improve both BipedalWalker and
Procgen without PopArt. Adding PopArt
\citep{vanhasselt2016popart,hessel2019popartmultitask} to the same multihead critic instead
reduces the aggregate normalized training gain on the pinned levels from $+24.2\%$ to $+10.0\%$
(Table~\ref{tab:popart}). Explicit target normalization is therefore not required for the
conditioning gain and does not explain the multihead improvement. The consistent factor
across the successful variants is the ability to fit a distinct value target for each
environment.

\paragraph{Other conditioning.}
On BipedalWalker we test two further conditioning schemes
(Appendix Figure~\ref{fig:bipedal-conditioning-location}). The first removes the shared
representation entirely and assigns each environment its own value network: it learns much
more slowly than FiLM and multihead, indicating that a shared critic representation remains
useful even when value targets differ. The second asks whether conditioning the actor also
helps training: supplying the environment index to the actor on top of the multihead critic
collapses both training and test return. These results favor the minimal asymmetric
intervention used in our experiments: retain representation sharing in the critic, separate
its environment-specific predictions, and leave the actor unconditioned.

\section{Discussion, limitations, and conclusion}
\label{sec:discussion}

\paragraph{Statistical sharing versus target separation.}
Conditioning is most useful when environment identity explains substantial variation in
value. At the population level, the conditional critic classes contain the shared critic as
a special case, so their optimal squared prediction error cannot be larger under the same
data distribution. With finite data, however, estimating environment-dependent components
can increase estimation error when values are already similar or some environments are
sampled infrequently. The practical tradeoff is therefore between statistical sharing and
separation of conflicting value targets. FiLM and multihead retain a shared representation
while allowing the final prediction to separate. Empirically, the early transient in the
compact CartPole study does not become a persistent fitting bottleneck on the larger
benchmarks, as shown by the value loss curves in
Appendix~\ref{app:value-loss-diagnostics}.

\paragraph{Mismatch severity.}
The number of environments alone does not determine mismatch severity. At a state $s$ visible
to the critic, the average squared mismatch is
$\E[(e_Z^\pi(s))^2\mid S=s]=\Var(V_Z^\pi(s)\mid S=s)$: many tightly clustered
environments may be benign, whereas two separated values can suffice. This average can also
hide a rarely sampled environment with a large pointwise offset, producing infrequent but
large miscentered updates.

\paragraph{Scope of the theory.} The propositions do not establish a uniform finite-horizon ordering in expected optimal-arm probability or return. Instead, under their stated conditions, they isolate a clean pathwise mechanism: shared centering produces recurrent reversed updates, whereas conditional centering eventually rules them out. 
We deliberately study illustrative deterministic bandits sampled across a
fixed set of environments. First, their baselines are \emph{oracle}: in practice, both the
environment-specific values and their shared marginal must be learned from data. Second,
deterministic rewards make the branch signs pathwise; stochastic returns preserve the offset
identity but not every realized sign. Third, the common optimal arm removes genuine conflict
between environments, and the fixed environment mixture removes policy-dependent changes in
environment sampling.
Deep RL adds state visitation, function approximation, learned critics, GAE, and PPO
transformations, so the experiments test whether the explanatory mechanism transfers rather
than verify a general MDP convergence theorem. The mechanism itself is broader: it can arise
whenever parallel environments produce different continuation values for inputs that the
critic maps to the same representation.

\paragraph{Scope of the method.}
Our experiments use recurring environments with stable logged identifiers. This is a
minimal diagnostic intervention, not a universal conditioning scheme: continuous or unseen
variants may instead require simulator parameters or a learned or inferred representation
\citep{rakelly2019pearl}. The actor never receives the identifier, so deployment remains
unchanged. Conditioning addresses value mismatch, but not genuine task conflict, critic
estimation error, or every possible cost of privileged information.

\paragraph{Conclusion.}
We introduced value mismatch as a direct cause of poorer sampled learning dynamics when one
critic is shared across environments. Because the mechanism depends on value differences
rather than a particular domain, it motivates a simple intervention: condition only the
critic on the environment index. Across four benchmark families, this intervention improves
learning while preserving one shared actor. A promising direction is to identify where
mismatch is large and condition the critic only there, combining accurate conditional value
fitting with statistical sharing elsewhere.

\section*{Acknowledgements}
We thank Jincheng Mei for meaningful discussions. X. Liu thanks the Department of Medicine and the UF AI for Health Institute at the University of Florida for the support. We gratefully acknowledge the support
of NSF IIS-2313131, IIS-2332475, IIS-2543755, and the NSF Simons AI-Institute for the Sky (SkAI) via grants NSF
AST-2421845 and Simons Foundation MPS-AI00010513.
\section*{AI use statement}
Generative AI tools were used only to polish the writing of the manuscript and to
assist in verifying the correctness of the mathematical proofs. All research ideas,
claims, theoretical results, experiment designs, and analyses of results were conceived
and carried out by the authors. The authors take full responsibility
for all content and conclusions of the paper.


\section*{Reproducibility statement}
Section~\ref{sec:setting} specifies the theoretical setting, assumptions, and update
rule. Complete proofs and auxiliary results appear in
Appendices~\ref{app:variance-view}--\ref{app:mismatchproofs}. Appendix~\ref{app:archdetails}
specifies the conditioning architectures and parameter counts, while
Appendix~\ref{app:expdetails} reports the environment pools for each benchmark,
hyperparameters, seed counts, evaluation protocols, and aggregation rules. Figure and
table captions state the number of seeds and the uncertainty convention used for the
reported results.

\newpage

\bibliography{reference}
\bibliographystyle{iclr2027_conference}
\newpage
\appendix
\section{What variance can and cannot explain}
\label{app:variance-view}

Classical baseline theory asks how an action-independent baseline reduces the
variance of a policy gradient estimator
\citep{weaver2001optimal,greensmith2004variance}.  This view gives a valid
aggregate guarantee: within a fixed information set, a score-weighted baseline
minimizes the trace of the gradient estimator covariance and maximizes the
standard smoothness lower bound. It does not identify which sampled combinations of environment and arm carry
the update, however, or how those branches change future sampling.  We first
state the variance guarantee and then show that its exact optimum can produce a
slower and less stable process over a finite horizon than the ordinary conditional value
baseline used in the main text.

Fix a policy parameter $\theta$.  All expectations and covariances below are
under the current on-policy joint law of $(S,Z,A,G)$, where
$A\mid(S,Z)\sim\pi_\theta(\cdot\mid S)$.  Let $J(\theta)$ denote the policy
objective and assume the policy gradient identity
$\E[G\psi]=\nabla J(\theta)$, where
$\psi:=\nabla_\theta\log\pi_\theta(A\mid S)$ and
$w:=\|\psi\|_2^2$.  Let $Y$ denote the information given to the baseline:
$Y_s:=S$ for a shared baseline and $Y_c:=(S,Z)$ for a conditional baseline.
Because the actor does not receive $Z$,
$\E[\psi\mid Y_s]=\E[\psi\mid Y_c]=0$.

For a baseline $B(Y)$, define
$\widehat g_B:=(G-B(Y))\psi$.  If $J$ is $L_J$-smooth, the ascent lemma gives
\[
 \E[J(\theta+\eta\widehat g_B)]
 \ge J(\theta)+\eta\|\nabla J(\theta)\|_2^2
 -\frac{L_J\eta^2}{2}\E\!\left[w(G-B(Y))^2\right].
\]
We denote the right-hand side by $\mathrm{LB}_{Y}(B;\eta)$ and write
$\mathrm{LB}_{s}:=\mathrm{LB}_{Y_s}$ and
$\mathrm{LB}_{c}:=\mathrm{LB}_{Y_c}$.

\begin{proposition}[Richer baseline information improves the variance certificate]
\label{prop:baseline-max-lower-bound}
Suppose $\E[wG^2]<\infty$ and
$0<\E[w\mid Y]<\infty$ almost surely.  Among all baselines measurable with
respect to $Y$ and satisfying $\E[wB(Y)^2]<\infty$, the unique minimizer up to
almost-sure equality is
\[
 B_Y^{\mathrm{mv}}(Y)
 =\frac{\E[wG\mid Y]}{\E[w\mid Y]}.
\]
Equivalently, $B_Y^{\mathrm{mv}}$ minimizes
$\E[w(G-B(Y))^2]$, minimizes the covariance trace
$\operatorname{tr}\Cov(\widehat g_B)$, and maximizes
$\mathrm{LB}_{Y}(B;\eta)$ for every $\eta>0$.  Let
$B_s^{\mathrm{mv}}$ and $B_c^{\mathrm{mv}}$ denote the optima under $Y_s$ and
$Y_c$.  Then
\[
 \mathrm{LB}_{c}(B_c^{\mathrm{mv}};\eta)
 -\mathrm{LB}_{s}(B_s^{\mathrm{mv}};\eta)
 =\frac{L_J\eta^2}{2}\,
 \E\!\left[w\bigl(B_c^{\mathrm{mv}}(S,Z)
                 -B_s^{\mathrm{mv}}(S)\bigr)^2\right].
\]
\end{proposition}

\begin{proof}
The baseline has zero mean score contribution, so every admissible $B$ gives
$\E[\widehat g_B]=\nabla J(\theta)$.  Conditional Cauchy--Schwarz gives
\[
 \E[w\mid Y](B_Y^{\mathrm{mv}})^2
 =\frac{\E[wG\mid Y]^2}{\E[w\mid Y]}
 \le \E[wG^2\mid Y],
\]
so $B_Y^{\mathrm{mv}}$ is admissible.  Moreover,
$\E[w(G-B_Y^{\mathrm{mv}})\mid Y]=0$.  Completing the square conditionally gives
\[
 \E[w(G-B)^2\mid Y]
 =\E[w(G-B_Y^{\mathrm{mv}})^2\mid Y]
 +\E[w\mid Y](B-B_Y^{\mathrm{mv}})^2.
\]
This proves optimality and uniqueness.  Because $Y_c$ refines $Y_s$, the same
orthogonality gives
\[
 \E[w(G-B_s^{\mathrm{mv}})^2]
 =\E[w(G-B_c^{\mathrm{mv}})^2]
 +\E[w(B_c^{\mathrm{mv}}-B_s^{\mathrm{mv}})^2],
\]
which proves the gap identity.  Finally, all admissible estimators have the same
mean, and therefore
\[
 \operatorname{tr}\Cov(\widehat g_B)
 =\E[w(G-B)^2]-\|\nabla J(\theta)\|_2^2.
\]
\end{proof}

\paragraph{Relation to value mismatch.}
Consider the softmax policy with one logit per arm in
Algorithm~\ref{alg:eispg}, two equally
likely environments $Z\in\{E,H\}$, and two arms $(a^*,i)$ with
$r_z=(d_z,0)$, where $d_E,d_H>0$.  Write $p:=\pi(a^*)$.  The squared score
norms are $w(a^*)=2(1-p)^2$ and $w(i)=2p^2$.  Put
$\bar d:=(d_E+d_H)/2$.  The average objective is $J(\theta)=\bar d p$, and
its Hessian with respect to the two logits is
\[
 \nabla^2J(\theta)
 =\bar d\,p(1-p)(1-2p)
 \begin{bmatrix}1&-1\\-1&1\end{bmatrix}.
\]
Consequently,
\[
 \sup_\theta\|\nabla^2J(\theta)\|_{\mathrm{op}}
 =2\bar d\max_{p\in[0,1]}|p(1-p)(1-2p)|
 =\frac{\bar d}{3\sqrt3}.
\]
Thus $L_J=\bar d/(3\sqrt3)$ is the smallest global Euclidean smoothness
constant for this parameterization. Proposition~\ref{prop:baseline-max-lower-bound}
then gives
\[
 B_c^{\mathrm{mv}}(z)=(1-p)d_z,
 \qquad
 B_s^{\mathrm{mv}}=(1-p)\bar d.
\]
The exact gain from the richer information set is
\[
 \mathrm{LB}_{c}(B_c^{\mathrm{mv}};\eta)
 -\mathrm{LB}_{s}(B_s^{\mathrm{mv}};\eta)
 =\frac{\bar d\eta^2}{12\sqrt3}\,
 p(1-p)^3(d_E-d_H)^2.
\]
Here $V_z^\pi=pd_z$, $\bar V^\pi=p\bar d$, and
$e_z^\pi:=V_z^\pi-\bar V^\pi$.  Hence the same difference equals
$\frac{\bar d\eta^2}{3\sqrt3}((1-p)^3/p)\E[(e_Z^\pi)^2]$.
At a fixed policy and data distribution,
the variance certificate therefore improves quadratically with the value mismatch.

The exact baseline that minimizes the covariance trace is not generally the ordinary value prediction used
by actor--critic methods:
\[
 B_Y^{\mathrm{mv}}
 =\E[G\mid Y]
 +\frac{\Cov(w,G\mid Y)}{\E[w\mid Y]}.
\]
The correction appears because $w=\|\mathbf e_A-\pi\|_2^2$ gives more weight to actions
with larger score norms. An ordinary value critic ignores this
action-dependent weight and minimizes the simpler prediction error
$\E[(G-B(Y))^2]$.  It is therefore an unweighted surrogate for the exact
minimizer of the covariance trace, but it still has a precise variance interpretation. Define
$B^{\mathrm{shared}}(S):=\E[G\mid S]$,
$B^{\mathrm{cond}}(S,Z):=V_Z^\pi(S):=\E[G\mid S,Z]$, and
$e_Z^\pi(S):=V_Z^\pi(S)-B^{\mathrm{shared}}(S)$.  Conditional expectation gives
\[
 \E[(G-B^{\mathrm{shared}}(S))^2]
 =\E[(G-B^{\mathrm{cond}}(S,Z))^2]+\E[(e_Z^\pi(S))^2].
\]
Moreover, $w\le2$ for this parameterization, so
\[
 \E[w(G-B)^2]\le2\E[(G-B)^2].
\]
Within either information set, ordinary value prediction therefore minimizes a
valid upper bound on the second moment in the smoothness certificate.  Replacing the shared value by the
conditional value tightens the resulting lower bound by exactly
$L_J\eta^2\E[(e_Z^\pi(S))^2]$. Larger mismatch thus strengthens the conventional
variance argument for conditioning. This aggregate certificate, however,
still does not reveal how the update is allocated across sampled branches or
which stochastic process follows a better path over a finite horizon.

\paragraph{A numerical counterexample.}
Figure~\ref{fig:minvar-counterexample} uses two equally likely environments with
three arms $(a^*,a_2,a_3)$, rewards
$r_E=(1,0.8,0)$ and $r_H=(0.5,0.4,0)$, common initialization
$\pi_0^{\mathrm{mv}}=\pi_0^{\mathrm{cond}}=(0.5,0.4,0.1)$, and $\eta=10$.
We run two copies of the score update in Algorithm~\ref{alg:eispg}.  Each copy recomputes its own
baseline from its current policy at every round:
\[
 B_t^{\mathrm{mv}}(z)
 =\frac{\sum_a\pi_t^{\mathrm{mv}}(a)w_t^{\mathrm{mv}}(a)r_z(a)}
        {\sum_a\pi_t^{\mathrm{mv}}(a)w_t^{\mathrm{mv}}(a)},
 \qquad
 B_t^{\mathrm{cond}}(z)
 =V_z^{\pi_t^{\mathrm{cond}}}
 =\sum_a\pi_t^{\mathrm{cond}}(a)r_z(a),
\]
where
$w_t^{\mathrm{mv}}(a):=\|\mathbf e_a-\pi_t^{\mathrm{mv}}\|_2^2$.
With one logit per arm, the first baseline exactly minimizes the covariance trace;
the second is the ordinary conditional value.

In environment $E$, their initial values are $0.704$ and $0.820$. The
covariance trace minimizer therefore lies below $r_E(a_2)=0.8$ and
reinforces a sample of $a_2$,
whereas the conditional value lies above it and suppresses the same sample.  That
branch sends $\pi(a^*)$ from $0.5$ to $0.279$ and $0.547$, respectively.  Yet the
initial covariance trace is smaller under that baseline:
$0.0423$ versus $0.0472$.

\begin{center}
\begin{minipage}{\linewidth}
  \centering
  \includegraphics[width=\linewidth]{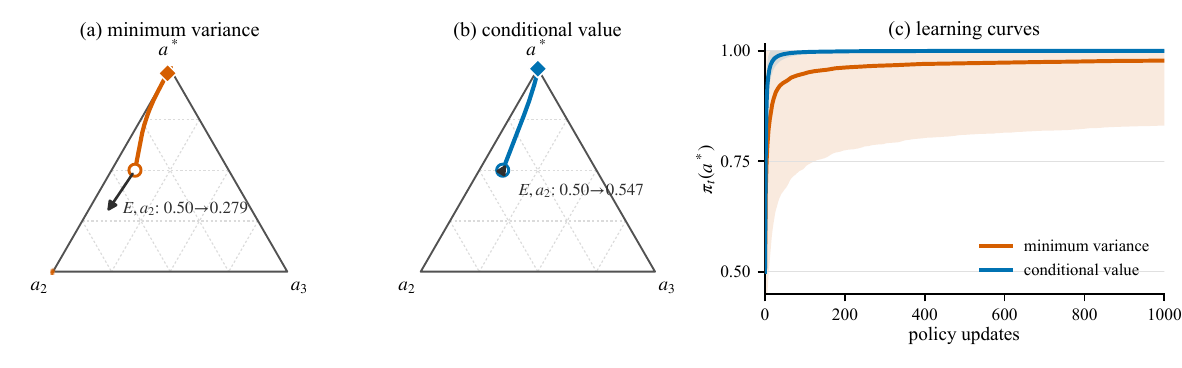}
  \captionof{figure}{Minimizing the covariance trace need not produce a better
  process over a finite horizon. The two environments share the same optimal arm and are
  sampled equally; $r_H=0.5r_E$, $\pi_0=(0.5,0.4,0.1)$, and $\eta=10$.
  (a,b)~Terminal policies at $T=1{,}000$ (light points; $4{,}000$ of
  $20{,}000$ shown) and ensemble mean trajectories (solid curves) from the common
  initialization (circle) to the terminal mean (diamond).  The arrows show the
  same sampled branch at the initial policy: sampling $a_2$ in the easy
  environment $E$ moves away from $a^*$ under the baseline minimizing trace
  covariance but toward $a^*$ under the conditional value. The
  initial covariance traces are $0.0423$ and $0.0472$, respectively.
  (c)~Optimal arm probability over all $20{,}000$ trajectories
  (mean$\,\pm\,$1 s.d.; common random numbers).  Both baselines are recomputed
  from their own current policies at every update. This compares the means at a fixed horizon
  and does not claim pathwise or asymptotic dominance.}
  \label{fig:minvar-counterexample}
\end{minipage}
\end{center}

Over $20{,}000$ trajectories with common random numbers, after $1{,}000$ updates
the process using the covariance trace minimizer reaches $0.9777\pm0.1468$, while the conditional
value process reaches $0.9998\pm0.0004$ (mean$\,\pm\,$1 s.d.).  The conditional
ensemble mean crosses $0.99$ after $38$ updates; the mean under the covariance trace minimizer remains
below $0.99$ through update $1{,}000$.  The second arm remains the policy winner
in $2.2\%$ of runs using that baseline and in none of the conditional value runs.
This is an ensemble comparison rather than pathwise dominance, but it shows that
minimizing the local covariance trace can create a substantially longer and less
stable committal transient.

The example exposes what the variance summary discards.  At initialization,
score weighting reduces the relative weight of the most probable optimal arm and
more than doubles the relative weight of the rare arm with zero reward, pulling the
covariance trace minimizer below $r_E(a_2)$. Baseline placement then changes which
sampled arm is reinforced, and that change feeds back into future action sampling.
The two processes still share the same asymptotic destination.  Indeed, after
averaging over $Z$, each has effective arm means
$\bar r(a)-c_t^B$, where
$c_t^B:=\sum_zq_zB_t^B(z)$ is bounded and predictable; hence
Lemma~\ref{lem:common shift} applies.  The example complements
\citet{chung2021beyond}, whose stronger convergence counterexample uses a
natural policy gradient process with three arms.  Variance remains a useful aggregate
certificate, but it cannot explain the branch allocation that governs the sampled
learning path.

\section{Proof of Proposition~\ref{prop:convergence}}
\label{app:prop1}

This section gives a complete convergence proof for the predictable
reward laws that depend on the environment and are induced by
Algorithm~\ref{alg:eispg}. The proof architecture is adapted from the exploration, barrier,
and elimination arguments of \citet{robertson2025reinforce}, including their use of
Freedman's inequality to prove divergence. All ingredients needed here, including the rate
for tail averages, are stated and proved directly in this appendix.

\paragraph{Proof convention.}
The shaded \emph{Comment} boxes mark the steps at which we adapt the proof architecture
of \citet{robertson2025reinforce} to our multi-environment common-shift setting. The
intervening probabilistic argument is reproduced self-contained.

\paragraph{Proof idea.}
The argument has three main steps.  First, after averaging over the sampled
environment, every baseline induces effective arm means of the form
$\bar r(a)-c_t$, where the same predictable scalar $c_t$ is subtracted from every
arm.  This shift changes the sampled noise but leaves every pairwise arm gap
fixed, and all three baselines induce the same conditional mean update.  Second,
conditional Borel--Cantelli gives infinite exploration.  The optimal logit has
positive mean drift; infinite exploration makes its cumulative drift diverge,
and stopped concentration prevents the noise from canceling that drift, so the
optimal logit tends to $+\infty$.  Third, a barrier argument reaches an arbitrarily
large simultaneous gap and then traps the process in a region where every rival
has strictly negative mean drift.  Continued exploration makes each accumulated
negative drift diverge, driving every rival logit to $-\infty$.  Consequently,
$\pi_t(a^*)\to1$ almost surely for all three baselines. Finally, the same
compensator argument yields, without requiring $c_t$ to settle, a pathwise
$O((1+\log N)/N)$ bound on average suboptimality over a random tail after the
transient; it is neither a last iterate nor an expected rate.

\paragraph{Effective reward law.}
Let $\mathcal F_t$ contain the complete history before $(Z_t,a_t)$ is sampled, so
$\pi_t$ is $\mathcal F_t$-measurable. Because
Algorithm~\ref{alg:eispg} samples $Z_t$ independently
from $q$ and the actor does not observe $Z_t$, conditionally on $\mathcal F_t$ the draws
$Z_t\sim q$ and $a_t\sim\pi_t$ are independent.  Define
\[
Y_t^B:=r_{Z_t}(a_t)-B_t^B,
\qquad
\bar r(a):=\sum_zq_zr_z(a),
\qquad
\bar V^{\pi_t}:=\sum_a\pi_t(a)\bar r(a).
\]
Also put $R:=\max_{z,a}|r_z(a)|<\infty$.

\begin{lemma}[Exact reduction to a common shift]
\label{lem:effective-common shift}
For $B\in\{0,\mathrm{shared},\mathrm{cond}\}$, the conditional law of $Y_t^B$
given $(\mathcal F_t,a_t=a)$ is predictable, supported on $[-2R,2R]$, and has mean
\begin{equation}
m_t^B(a):=\E[Y_t^B\mid\mathcal F_t,a_t=a]
=\bar r(a)-c_t^B,
\quad
c_t^0=0,
\quad
c_t^{\mathrm{shared}}=c_t^{\mathrm{cond}}=\bar V^{\pi_t}.
\label{eq:effective-arm-means}
\end{equation}
Consequently, every pairwise gap is fixed,
$m_t^B(a)-m_t^B(i)=\bar r(a)-\bar r(i)$, and
\begin{equation}
\E[\Delta\theta_t(a)\mid\mathcal F_t]
=\eta\pi_t(a)\bigl(\bar r(a)-\bar V^{\pi_t}\bigr).
\label{eq:app-pooled-drift}
\end{equation}
\end{lemma}

\begin{proof}
The zero and shared cases follow directly from their definitions.  For the conditional
baseline, conditional independence of $Z_t$ and $a_t$ gives
\[
\E[Y_t^{\mathrm{cond}}\mid\mathcal F_t,a_t=a]
=\sum_zq_z\bigl(r_z(a)-V_z^{\pi_t}\bigr)
=\bar r(a)-\bar V^{\pi_t}.
\]
\adaptationnote{
The conditional oracle baseline is not a common shift on a realized environment
branch. Independence of $Z_t$ and $a_t$ makes its mean conditional on the arm exactly
the same predictable common shift after averaging over the environment index.}
Substitution into the score update proves \eqref{eq:app-pooled-drift}.  Every oracle
value is a convex combination of the finite reward table, so $|Y_t^B|\le2R$ in all
three cases.
\end{proof}

Throughout the remainder of the proof, write
$\E_t[\cdot]:=\E[\cdot\mid\mathcal F_t]$ and
$\Var_t(\cdot):=\Var(\cdot\mid\mathcal F_t)$; for a stopping time $\tau$,
$\Prob_\tau(\cdot):=\Prob(\cdot\mid\mathcal F_\tau)$.

We first record a stopped concentration lemma that is uniform over time. It is a direct
Freedman--Bernstein argument, included so that the buffers used below have explicit
probabilities and do not depend on an external bandit result.

\begin{lemma}[Stopped Freedman bound under drift dominance]
\label{lem:stopped-freedman}
Let $\tau$ be a stopping time finite almost surely and let $(X_t)_{t\ge\tau}$ be an
adapted process taking real values with
\[
X_{t+1}-X_t=s_t+\xi_{t+1},
\qquad
\E[\xi_{t+1}\mid\mathcal F_t]=0,
\qquad
|\xi_{t+1}|\le b,
\]
where $0\le b<\infty$ is deterministic. Suppose that, for some deterministic
$0\le c<\infty$ with $b+c>0$, one of the following two conditions holds
at every active step:
\begin{align}
s_t=\nu_t\ge0,
&\qquad \E[\xi_{t+1}^2\mid\mathcal F_t]\le c\nu_t,
\label{eq:freedman-positive-case}\\
s_t=-\nu_t\le0,
&\qquad \E[\xi_{t+1}^2\mid\mathcal F_t]\le c\nu_t,
\label{eq:freedman-negative-case}
\end{align}
where $\nu_t\ge0$ is predictable. For $q\in(0,1)$,
define
\begin{equation}
D_q(b,c):=2(b+c)\log(1/q).
\label{eq:freedman-explicit-buffer}
\end{equation}
Conditionally on $\mathcal F_\tau$, under \eqref{eq:freedman-positive-case},
\begin{equation}
\Prob_\tau\!\left(
\inf_{n\ge\tau}(X_n-X_\tau)<-D_q(b,c)
\right)\le q,
\label{eq:freedman-lower-buffer}
\end{equation}
whereas under \eqref{eq:freedman-negative-case},
\begin{equation}
\Prob_\tau\!\left(
\sup_{n\ge\tau}(X_n-X_\tau)>D_q(b,c)
\right)\le q.
\label{eq:freedman-upper-buffer}
\end{equation}
Moreover, if $\Lambda_n:=\sum_{t=\tau}^{n-1}\nu_t\to\infty$, then $X_n\to+\infty$
under \eqref{eq:freedman-positive-case} and $X_n\to-\infty$ under
\eqref{eq:freedman-negative-case}, almost surely.  More precisely, writing
$M_n:=\sum_{t=\tau}^{n-1}\xi_{t+1}$, almost surely there is a finite random index
$N$ such that
\begin{equation}
|M_n|\le \frac12\Lambda_n\qquad(n\ge N).
\label{eq:freedman-eventual-tracking}
\end{equation}
The conclusions remain valid up to
a further stopping time after multiplying every increment by the predictable indicator
that this stopping time has not yet occurred.
\end{lemma}

\begin{proof}
For $0<\lambda<3/b$, the conditional Bernstein bound for the martingale difference is
\begin{equation}
\E_t[e^{\lambda\xi_{t+1}}]
\le\exp\!\left(
\frac{\lambda^2}{2(1-\lambda b/3)}
\E_t[\xi_{t+1}^2]
\right),
\label{eq:conditional-bernstein-mgf}
\end{equation}
and the same inequality holds for $-\xi_{t+1}$.  Take
$\lambda:=1/[2(b+c)]$.  Then $\lambda b\le1/2$, $\lambda c\le1/2$, and, writing
$\psi_b(\lambda):=\lambda^2/[2(1-\lambda b/3)]$,
\begin{equation}
\psi_b(\lambda)c\le\frac{3}{10}\lambda<\frac12\lambda.
\label{eq:bernstein-slack}
\end{equation}
For $M_n:=\sum_{t=\tau}^{n-1}\xi_{t+1}$ and
$\Lambda_n:=\sum_{t=\tau}^{n-1}\nu_t$, the exponential process
\[
\exp\!\left(
\lambda M_n-\psi_b(\lambda)
\sum_{t=\tau}^{n-1}\E_t[\xi_{t+1}^2]
\right)
\]
is a nonnegative supermartingale starting at one; so is the process with $M_n$
replaced by $-M_n$.  Ville's inequality and
\eqref{eq:bernstein-slack} therefore give, for every $x>0$,
\begin{equation}
\Prob_\tau(\exists n\ge\tau:M_n\ge\Lambda_n+x)\le e^{-\lambda x},
\qquad
\Prob_\tau(\exists n\ge\tau:-M_n\ge\Lambda_n+x)\le e^{-\lambda x}.
\label{eq:ville-two-sided}
\end{equation}
Since $X_n-X_\tau=\Lambda_n+M_n$ in the positive case and
$X_n-X_\tau=-\Lambda_n+M_n$ in the negative case, taking
$x=\lambda^{-1}\log(1/q)$ proves
both \eqref{eq:freedman-lower-buffer} and \eqref{eq:freedman-upper-buffer}.

For the divergence statements, \eqref{eq:bernstein-slack} leaves the deterministic
slack
$\kappa:=\lambda/2-\psi_b(\lambda)c>0$.  Ville's inequality gives, for each sign,
\[
\Prob_\tau(\exists n:\Lambda_n\ge u,\ M_n\ge\Lambda_n/2)
\le e^{-\kappa u},
\qquad
\Prob_\tau(\exists n:\Lambda_n\ge u,\ -M_n\ge\Lambda_n/2)
\le e^{-\kappa u}
\]
for every $u>0$.  Apply both bounds at $u=1,2,\ldots$.  Their failure probabilities
are summable.  By Borel--Cantelli, almost surely only finitely many integer levels $u$
admit an index $n$ with $\Lambda_n\ge u$ and either displayed deviation.  On
$\Lambda_n\to\infty$, such a deviation at arbitrarily late indices would cross arbitrarily
large integer levels and contradict this conclusion.  Thus eventually
$|M_n|<\Lambda_n/2$.  Hence
$\Lambda_n+M_n\to+\infty$ in the positive case and
$-\Lambda_n+M_n\to-\infty$ in the negative
case.  Multiplying by a predictable stopping indicator preserves all conditional
mean, variance, and increment bounds, proving the final assertion.
\end{proof}

\adaptationnote{
For $G(\pi):=\operatorname{Diag}(\pi)-\pi\pi^\top$, every predictable common
shift satisfies
\[
G(\pi_t)(\mu-c_t\mathbf 1)=G(\pi_t)\mu,
\qquad G(\pi_t)\mathbf 1=0.
\]
Thus the nonsettling offset changes the martingale noise law but never the mean
direction or an arm gap. The coordinate formulas and offset-uniform variance bounds
are derived below wherever they are used.}

\begin{lemma}[Elimination for general $K$ under a predictable common shift]
\label{lem:general-k-elimination}
Assume $K\ge2$, $0<\eta<\infty$, and
\begin{equation}
\theta_{t+1}=\theta_t+\eta Y_t(\mathbf e_{a_t}-\pi_t),
\qquad
\pi_t=\operatorname{softmax}(\theta_t),
\qquad
|Y_t|\le C,
\label{eq:elimination-score-update}
\end{equation}
where $\mathbf e_a$ is the $a$th standard basis vector,
$a_t\mid\mathcal F_t\sim\pi_t$, and
\begin{equation}
\E[Y_t\mid\mathcal F_t,a_t=a]=\mu(a)-c_t.
\label{eq:elimination-common shift}
\end{equation}
Here $\mu$ is fixed, $c_t$ is predictable, and $\mu$ has the unique maximizer $a^*$.
Suppose, almost surely, that
\begin{equation}
\theta_t(a^*)\to+\infty,
\qquad
\sum_{t=0}^\infty\pi_t(i)=\infty\quad(i\ne a^*).
\label{eq:elimination-inputs}
\end{equation}
Then
\begin{equation}
\theta_t(i)\to-\infty\quad(i\ne a^*),
\qquad
\pi_t(a^*)\to1
\quad\text{almost surely}.
\label{eq:elimination-conclusion}
\end{equation}
\end{lemma}

\begin{proof}
Let $\mathcal R:=\{1,\ldots,K\}\setminus\{a^*\}$ and define the fixed optimal gaps
\[
d_i:=\mu(a^*)-\mu(i),
\qquad
\delta:=\min_{i\in\mathcal R}d_i>0,
\qquad
L:=\max_{i\in\mathcal R}d_i<\infty.
\]
\adaptationnote{
Robertson et al. organize elimination by reward tiers so that ties and multiple
optimal arms are allowed. Here the shared optimum is strict, so one simultaneous
optimal--rival cone controls every rival without a tier induction. The proof below
defines the required stopping times and good events explicitly.}
Write $p_t^*:=\pi_t(a^*)$ and
$\bar d_t:=\sum_{i\in\mathcal R}\pi_t(i)d_i$.  The common shift cancels exactly,
giving
\begin{align}
m_t(a^*)&:=\E_t[\Delta\theta_t(a^*)]
=\eta p_t^*\bar d_t,
\label{eq:rankfree-optimal-drift}\\
m_t(i)&:=\E_t[\Delta\theta_t(i)]
=\eta\pi_t(i)(\bar d_t-d_i).
\label{eq:rankfree-rival-drift}
\end{align}
The score form and $|Y_t|\le C$ also give, for every arm $a$,
\begin{align}
|\Delta\theta_t(a)|&\le B_0:=\eta C,
\label{eq:rankfree-increment-bound}\\
\E_t|\Delta\theta_t(a)|
&\le2\eta C\pi_t(a)(1-\pi_t(a)),\nonumber\\
\Var_t(\Delta\theta_t(a))
&\le\eta^2C^2\pi_t(a)(1-\pi_t(a)).
\label{eq:rankfree-moments}
\end{align}

\emph{Step 1: arbitrary simultaneous gaps are reached.}
We claim that, for every deterministic $G>0$, almost surely there is a finite time
$T_G$ such that
\begin{equation}
\theta_{T_G}(a^*)-\max_{i\in\mathcal R}\theta_{T_G}(i)\ge G.
\label{eq:simultaneous-gap}
\end{equation}
Fix $G>0$ and $q\in(0,1)$.  Choose a deterministic $M>1$ satisfying
\begin{equation}
M\ge\frac{2(L-\delta)}{\delta},
\qquad
\frac{M^2}{1+2M}\ge\frac{4C}{\delta},
\label{eq:gap-M-choice}
\end{equation}
and set
\begin{equation}
U_2:=B_0+\log((K-1)M),
\qquad
U_1:=U_2+G.
\label{eq:gap-levels}
\end{equation}
From \eqref{eq:rankfree-optimal-drift} and \eqref{eq:rankfree-moments},
\begin{equation}
m_t(a^*)\ge\eta\delta p_t^*(1-p_t^*)\ge0,
\qquad
\Var_t(\Delta\theta_t(a^*))
\le c_*m_t(a^*),
\quad c_*:=\frac{\eta C^2}{\delta}.
\label{eq:rankfree-optimal-self-bound}
\end{equation}
The centered optimal increment has magnitude at most $b:=2B_0$.  Define
$D_*(q):=D_q(b,c_*)$ and the stopping time
\begin{equation}
\tau:=\inf\{t:\theta_t(a^*)\ge U_1+D_*(q)\}.
\label{eq:gap-start-time}
\end{equation}
It is finite almost surely by \eqref{eq:elimination-inputs}.  Lemma
\ref{lem:stopped-freedman}, restarted at $\tau$, shows that
\begin{equation}
\mathcal E_q:=\{\theta_t(a^*)\ge U_1\text{ for every }t\ge\tau\}
\label{eq:gap-optimal-good-event}
\end{equation}
has conditional probability at least $1-q$.

Define $\Phi_t:=\sum_{i\in\mathcal R}[\theta_t(i)]_+$, where
$[x]_+:=\max\{x,0\}$.  Consider a time $t\ge\tau$ at which
$\theta_t(a^*)\ge U_1$ and at least one rival logit is at least $U_2$.  Partition
\[
\mathcal B_t:=\{i\in\mathcal R:\theta_t(i)\ge B_0\},
\qquad
\mathcal C_t:=\mathcal R\setminus\mathcal B_t,
\]
and write $P_B:=\pi_t(\mathcal B_t)$ and $P_C:=\pi_t(\mathcal C_t)$.  For
$i\in\mathcal B_t$, a step of magnitude at most $B_0$ cannot cross below zero, so
the expected increment of $[\theta_t(i)]_+$ equals $m_t(i)$.  For
$i\in\mathcal C_t$, the $1$-Lipschitz property of $[\cdot]_+$ and
\eqref{eq:rankfree-moments} give an expected increase of at most
$2\eta C\pi_t(i)$.  Furthermore, \eqref{eq:rankfree-rival-drift} gives
\begin{align}
\frac1\eta\sum_{i\in\mathcal B_t}m_t(i)
&=-(1-P_B)\sum_{i\in\mathcal B_t}\pi_t(i)d_i
+P_B\sum_{i\in\mathcal C_t}\pi_t(i)d_i\nonumber\\
&\le-\delta p_t^*P_B+(L-\delta)P_BP_C.
\label{eq:positive-part-drift}
\end{align}
If $P_C=0$, this is nonpositive.  Otherwise, at least one rival is above $U_2$,
whereas all logits in $\mathcal C_t$ are below $B_0$.  Hence
\begin{equation}
\frac{P_B}{P_C}\ge M,
\qquad
\frac{p_t^*}{P_C}\ge M.
\label{eq:gap-softmax-ratios}
\end{equation}
Indeed, each denominator has at most $K-1$ terms, and
$\theta_t(a^*)\ge U_1>U_2$.  If $x:=p_t^*/P_C$ and $y:=P_B/P_C$, normalization gives
$P_C=(1+x+y)^{-1}$ and therefore
\begin{equation}
\frac{p_t^*P_B}{P_C}
=\frac{xy}{1+x+y}
\ge\frac{M^2}{1+2M}
\ge\frac{4C}{\delta}.
\label{eq:gap-product-ratio}
\end{equation}
The first condition in \eqref{eq:gap-M-choice} also gives
$(L-\delta)P_BP_C\le(\delta/2)p_t^*P_B$.  Combining these bounds yields
\begin{equation}
\E_t[\Phi_{t+1}-\Phi_t]
\le\eta\left[-\delta p_t^*P_B+(L-\delta)P_BP_C+2CP_C\right]
\le0.
\label{eq:positive-parts-supermartingale}
\end{equation}

Let $\nu$ be the first $t\ge\tau$ at which either
$\theta_t(a^*)<U_1$ or every rival logit is below $U_2$.  The drift bound in
\eqref{eq:positive-parts-supermartingale} shows that
$\Phi_{t\wedge\nu}=\sum_{i\in\mathcal R}[\theta_{t\wedge\nu}(i)]_+$ is a
nonnegative supermartingale.  To handle its
random starting value formally, restrict to each event
$\{\tau=n,\Phi_n\le m\}\in\mathcal F_n$ for integers $n,m$ and then take their countable
union. The convergence theorem for nonnegative supermartingales shows that the stopped
process has a finite limit almost surely.  On $\mathcal E_q$, the first stopping
condition never occurs.  If the second occurs, \eqref{eq:gap-levels} gives
\eqref{eq:simultaneous-gap}; if it never occurs, the positive parts of all rival logits
are bounded while $\theta_t(a^*)\to+\infty$, and \eqref{eq:simultaneous-gap} again
holds at a finite later time.  Thus the probability of never attaining the gap is at
most $q$.  Letting $q\downarrow0$ proves the claim.

\emph{Step 2: the process is eventually trapped in a cone where every rival has strictly negative drift.}
Choose
\begin{equation}
\rho\in\left(\max\left\{1-\frac{\delta}{L},\frac1K\right\},1\right),
\qquad
\gamma:=\delta-L(1-\rho)>0.
\label{eq:strict-cone}
\end{equation}
Whenever $p_t^*\ge\rho$, \eqref{eq:rankfree-optimal-drift} and
\eqref{eq:rankfree-rival-drift} imply
\begin{equation}
m_t(a^*)\ge0,
\qquad
m_t(i)\le-\eta\gamma\pi_t(i)\quad(i\in\mathcal R).
\label{eq:cone-drifts}
\end{equation}
Together with \eqref{eq:rankfree-moments}, these give
\begin{equation}
\Var_t(\Delta\theta_t(a^*))\le c_*m_t(a^*),
\qquad
\Var_t(\Delta\theta_t(i))\le c_R[-m_t(i)],
\quad c_R:=\frac{\eta C^2}{\gamma}.
\label{eq:cone-self-bounds}
\end{equation}

Fix $\varepsilon\in(0,1)$, put $q:=\varepsilon/K$, and define
\begin{equation}
D_*:=D_q(b,c_*),
\qquad
D_R:=D_q(b,c_R),
\qquad
g_\rho:=\log\frac{(K-1)\rho}{1-\rho},
\qquad
G_\varepsilon:=g_\rho+D_*+D_R.
\label{eq:cone-buffers}
\end{equation}
By Step~1, the first time $T_\varepsilon$ such that
\begin{equation}
\theta_{T_\varepsilon}(a^*)-\theta_{T_\varepsilon}(i)
\ge G_\varepsilon
\qquad(i\in\mathcal R)
\label{eq:cone-entrance}
\end{equation}
is a stopping time finite almost surely. This inequality implies
$p_{T_\varepsilon}^*\ge\rho$.  Let
$\sigma:=\inf\{t\ge T_\varepsilon:p_t^*<\rho\}$ and define the stopped increments
$\widetilde{\Delta\theta_t(a)}:=\mathbf 1\{t<\sigma\}\Delta\theta_t(a)$.
The indicator is $\mathcal F_t$-measurable, and every active update lies in the cone,
so both \eqref{eq:cone-drifts} and \eqref{eq:cone-self-bounds} apply to the stopped process.
Lemma~\ref{lem:stopped-freedman}, restarted at
$T_\varepsilon$, and a union bound over all $K$ arms give an event
$\mathcal G_\varepsilon$ of conditional probability at least $1-\varepsilon$ on which,
simultaneously for all $n\ge T_\varepsilon$ and all rivals $i$,
\begin{equation}
\theta_{n\wedge\sigma}(a^*)
\ge\theta_{T_\varepsilon}(a^*)-D_*,
\qquad
\theta_{n\wedge\sigma}(i)
\le\theta_{T_\varepsilon}(i)+D_R.
\label{eq:cone-good-event}
\end{equation}
If $\sigma<\infty$ on this event, then
\eqref{eq:cone-entrance} and \eqref{eq:cone-good-event} imply
\[
\theta_\sigma(a^*)-\theta_\sigma(i)\ge g_\rho
\quad(i\in\mathcal R),
\]
and hence
\[
p_\sigma^*
=\frac{1}{1+\sum_{i\in\mathcal R}
e^{\theta_\sigma(i)-\theta_\sigma(a^*)}}
\ge\frac{1}{1+(K-1)e^{-g_\rho}}
=\rho,
\]
contradicting the definition of $\sigma$.  Therefore, with probability at least
$1-\varepsilon$, the process never leaves the cone after $T_\varepsilon$.  If
$\mathcal A$ denotes the event that there is no finite time after which the process remains in the cone,
then $\Prob(\mathcal A)\le\varepsilon$ for every $\varepsilon>0$.  Thus
\begin{equation}
\Prob(\exists T<\infty:\ p_t^*\ge\rho\text{ for every }t\ge T)=1.
\label{eq:eventual-cone}
\end{equation}

\emph{Step 3: eliminate every rival from deterministic restart times.}
The last entrance time in \eqref{eq:eventual-cone} need not be a stopping time.  For
each deterministic integer $n\ge0$, define instead
\[
\sigma_n:=\inf\{t\ge n:p_t^*<\rho\}.
\]
Fix $i\in\mathcal R$ and use the increments
$\mathbf 1\{t<\sigma_n\}\Delta\theta_t(i)$ from time $n$.  On
$\{\sigma_n=\infty\}$, \eqref{eq:cone-drifts} and \eqref{eq:cone-self-bounds}
give the case of negative drift in Lemma~\ref{lem:stopped-freedman}, and its compensator
satisfies
\[
\sum_{t=n}^\infty[-m_t(i)]
\ge\eta\gamma\sum_{t=n}^\infty\pi_t(i)
=\infty
\]
by \eqref{eq:elimination-inputs}.  Hence
$\theta_t(i)\to-\infty$ on $\{\sigma_n=\infty\}$, almost surely. The event in
\eqref{eq:eventual-cone} agrees almost surely with
$\bigcup_{n\ge0}\{\sigma_n=\infty\}$, whose probability is one; the countable union therefore proves
the limit for this rival.  There are finitely many rivals, so all limits hold
simultaneously. Together with $\theta_t(a^*)\to+\infty$, every margin between the optimal arm and a rival
diverges and $\pi_t(a^*)\to1$.
\end{proof}

The preceding rank-free barrier plays the role of the tier-elimination step in
earlier proofs for softmax bandits. We can now prove the abstract convergence
statement needed by Proposition~\ref{prop:convergence}.

\begin{lemma}[Bounded predictable common shifts]
\label{lem:common shift}
Let $(\mathcal F_t)_{t\ge0}$ be a filtration, let $\theta_t$ be
$\mathcal F_t$-measurable, and put $\pi_t=\operatorname{softmax}(\theta_t)$.  Suppose
$c_t$ is $\mathcal F_t$-measurable, $(a_t,Y_t)$ and hence $\theta_{t+1}$ are
$\mathcal F_{t+1}$-measurable, $a_t\mid\mathcal F_t\sim\pi_t$, and the conditional
law of $Y_t$ given $(\mathcal F_t,a_t=a)$ is an $\mathcal F_t$-measurable kernel.
Consider
\begin{equation}
\theta_{t+1}=\theta_t+\eta Y_t(\mathbf e_{a_t}-\pi_t),
\qquad |Y_t|\le C<\infty,
\label{eq:abstract-score-update}
\end{equation}
where
\begin{equation}
\E[Y_t\mid\mathcal F_t,a_t=a]=\mu(a)-c_t.
\label{eq:abstract-common shift}
\end{equation}
Here $\mu\in\mathbb R^K$ is fixed, $c_t$ is a predictable scalar, and $\mu$ has a
unique maximizer $a^*$.  For every fixed finite $\eta>0$ and finite initial logits,
\begin{equation}
\pi_t(a^*)\longrightarrow1
\qquad\text{almost surely}.
\label{eq:abstract-convergence}
\end{equation}
\end{lemma}

\begin{proof}
Put $p_t:=\pi_t(a^*)$.  The case $K=1$ is immediate, so suppose $K\ge2$.

\emph{Conservation and moment bounds.}
The score coordinates sum to zero, so
$\sum_a\theta_t(a)=\sum_a\theta_0(a)$ pathwise.  The predictable shift cancels from
the drift:
\begin{align}
\E_t[\Delta\theta_t(a)]
&=\eta\pi_t(a)\left(\mu(a)-\sum_j\pi_t(j)\mu(j)\right),
\label{eq:common shift-drift}\\
|\Delta\theta_t(a)|&\le\eta C,\nonumber\\
\E_t|\Delta\theta_t(a)|
&\le2\eta C\pi_t(a)(1-\pi_t(a)),\nonumber\\
\Var_t(\Delta\theta_t(a))
&\le\eta^2C^2\pi_t(a)(1-\pi_t(a)).
\label{eq:common shift-moments}
\end{align}
\emph{Exploration of all arms.}
L\'evy's conditional Borel--Cantelli lemma gives, simultaneously for all arms,
\begin{equation}
\{a_t=a\text{ only finitely often}\}
=\left\{\sum_t\pi_t(a)<\infty\right\}
\quad\text{almost surely}.
\label{eq:exploration-bc}
\end{equation}
Suppose an arm $a$ were sampled only finitely often.  After its last sample,
$|\Delta\theta_t(a)|\le\eta C\pi_t(a)$, so \eqref{eq:exploration-bc} implies that
$\theta_t(a)$ converges to a finite value.  At the same time $\pi_t(a)\to0$, hence
$\max_j\theta_t(j)\to+\infty$.  Logit conservation then forces some arm $b$ to have
$\liminf_t\theta_t(b)=-\infty$.  Every finitely sampled arm has a finite logit limit
by the preceding argument, so $b$ is sampled infinitely often.  As there are finitely
many arms, it is enough to fix a deterministic pair $(a,b)$ with these properties.

Conditional Borel--Cantelli gives
$\sum_t\pi_t(b)=\infty$ and $\sum_t\pi_t(a)<\infty$.  On the event
$\{\theta_t(b)\le\theta_t(a)\}$, we have $\pi_t(b)\le\pi_t(a)$; hence
\[
\sum_t\Prob(a_t=b,\ \theta_t(b)\le\theta_t(a)\mid\mathcal F_t)
\le\sum_t\pi_t(a)<\infty.
\]
Only finitely many such samples occur. Fix $\varepsilon>0$. Choose $T$ after the last sample of $a$ and
the last sample of $b$ made while $\theta_t(b)\le\theta_t(a)$, and so late that
$\eta C\sum_{t\ge T}\pi_t(a)<\varepsilon$.  Since
$\liminf_t[\theta_t(b)-\theta_t(a)]=-\infty$, choose $u\ge T$ with this margin below
$-2\varepsilon$.  Until the margin crosses zero, neither $a$ nor $b$ is sampled.  A
sample of any other arm can increase it by at most $\eta C\pi_t(a)$, so the total
possible increase after $u$ is less than $\varepsilon$.  The margin can never cross
zero, contradicting the fact that $b$ is sampled infinitely often but eventually
never while its margin is nonpositive.  Thus every arm is sampled infinitely often.

\emph{Divergence of the optimal logit.}
Let $\Delta_*:=\min_{i\ne a^*}[\mu(a^*)-\mu(i)]>0$.  From
\eqref{eq:common shift-drift} and \eqref{eq:common shift-moments},
\begin{align}
\E_t[\Delta\theta_t(a^*)]
&\ge\eta\Delta_*p_t(1-p_t),
\label{eq:optimal-positive-drift}\\
\Var_t(\Delta\theta_t(a^*))
&\le\frac{\eta C^2}{\Delta_*}\,
\E_t[\Delta\theta_t(a^*)].
\label{eq:optimal-var-drift}
\end{align}
Exploration and \eqref{eq:exploration-bc} imply
$\sum_tp_t=\sum_t(1-p_t)=\infty$.  If eventually $p_t\le1/2$, then
$p_t(1-p_t)\ge p_t/2$; if eventually $p_t\ge1/2$, then
$p_t(1-p_t)\ge(1-p_t)/2$.  Otherwise there are infinitely many upcrossings of
$1/2$. The aggregate log odds
\[
\log\frac{p_t}{1-p_t}
=\theta_t(a^*)-\log\sum_{i\ne a^*}e^{\theta_t(i)}
\]
changes by at most $2\eta C$ in one step.  An upcrossing therefore lands in the
compact interval
$[1/2,e^{2\eta C}/(1+e^{2\eta C})]$, on which $p(1-p)$ is bounded away from zero.
In all three cases,
\begin{equation}
\sum_t p_t(1-p_t)=\infty.
\label{eq:product-divergence}
\end{equation}
Apply Lemma~\ref{lem:stopped-freedman} to the optimal coordinate from time zero.
Its compensator diverges by \eqref{eq:optimal-positive-drift} and
\eqref{eq:product-divergence}, its variance is self-bounded by
\eqref{eq:optimal-var-drift}, and its centered increments have magnitude at most
$2\eta C$.  The divergence conclusion of the lemma gives
$\theta_t(a^*)\to+\infty$ almost surely.

The equivalence in \eqref{eq:exploration-bc} now gives
$\sum_t\pi_t(i)=\infty$ for every rival.  All hypotheses of
Lemma~\ref{lem:general-k-elimination} hold, so every rival logit tends to $-\infty$ and
\eqref{eq:abstract-convergence} follows.
\end{proof}

\adaptationnote{
The rate proof in Theorem~E.1 of Appendix~E in \citet{robertson2025reinforce} is
driven by the optimal-logit compensator, not by stationarity of every realized reward
law. Because offset cancellation makes it exactly equal to the fixed-$\mu$ stationary
bandit compensator, the same scalar mechanism can be proved directly without waiting
for $c_t$ to converge. The following argument also replaces the auxiliary deterministic
recursion cited in that appendix by a direct exponential potential calculation.}

\begin{corollary}[Rate for tail averages under a common shift]
\label{cor:common shift-rate}
Under Lemma~\ref{lem:common shift}, suppose $K\ge2$, and define
\[
\delta:=\min_{i\ne a^*}[\mu(a^*)-\mu(i)]>0,
\qquad
L:=\max_{i\ne a^*}[\mu(a^*)-\mu(i)],
\qquad
\kappa:=\frac{\eta\delta}{4}.
\]
On an event of probability one, there is a finite random integer $\tau$ such that,
simultaneously for every integer $T>\tau$,
\begin{equation}
\frac1{T-\tau}\sum_{t=\tau}^{T-1}
\left(\mu(a^*)-\sum_a\pi_t(a)\mu(a)\right)
\le
\frac{L}{\kappa(T-\tau)}
\log\!\left(1+(K-1)(e^\kappa-1)(T-\tau)\right).
\label{eq:common shift-rate-explicit}
\end{equation}
Consequently, the left-hand side is
$O((1+\log(1+T-\tau))/(T-\tau))$ pathwise.  For $K=1$, it is identically zero.
\end{corollary}

\begin{proof}
The case $K=1$ was separated in the statement, so assume $K\ge2$.  Softmax and the
score update are invariant under adding the same constant to every logit.  Subtracting
$\theta_0(a^*)\mathbf 1$ from the entire logit trajectory, we may therefore assume
without loss of generality that $\theta_0(a^*)=0$.

Put
\[
p_t:=\pi_t(a^*),
\qquad
\varepsilon_t:=1-p_t,
\qquad
g_t:=\E_t[\Delta\theta_t(a^*)],
\]
and define the cumulative conditional drift and its martingale remainder by
\[
\Xi_t:=\sum_{s=0}^{t-1}g_s,
\qquad
M_t:=\theta_t(a^*)-\Xi_t.
\]
The proof of Lemma~\ref{lem:common shift} gives
\begin{equation}
g_t\ge\eta\delta p_t\varepsilon_t,
\qquad
\Var_t(\Delta\theta_t(a^*))
\le\frac{\eta C^2}{\delta}g_t,
\label{eq:rate-drift-self-bound}
\end{equation}
and \eqref{eq:product-divergence} implies $\Xi_t\to\infty$.  Applying the quantitative
conclusion \eqref{eq:freedman-eventual-tracking} of
Lemma~\ref{lem:stopped-freedman} from time zero therefore shows that, almost surely,
\begin{equation}
M_t\ge-\frac12\Xi_t
\qquad\text{for all sufficiently large }t.
\label{eq:rate-martingale-tracking}
\end{equation}

The convergence and elimination parts of Lemma~\ref{lem:common shift} prove
$p_t\to1$ and $\theta_t(i)\to-\infty$ for every $i\ne a^*$. Hence, on the same event
of probability one, choose a finite random integer $\tau$ such that, for every
$t\ge\tau$, \eqref{eq:rate-martingale-tracking} holds, $p_t\ge1/2$, and
$\theta_t(i)\le0$ for all $i\ne a^*$.  Define
\[
S_t:=\sum_{s=\tau}^{t-1}\varepsilon_s,
\qquad t\ge\tau.
\]
Since $g_s\ge0$, \eqref{eq:rate-drift-self-bound} yields
\[
\Xi_t
\ge\sum_{s=\tau}^{t-1}g_s
\ge\frac{\eta\delta}{2}S_t.
\]
Using $\theta_0(a^*)=0$ and \eqref{eq:rate-martingale-tracking}, we obtain
\begin{equation}
\theta_t(a^*)=\Xi_t+M_t
\ge\frac12\Xi_t
\ge\kappa S_t.
\label{eq:rate-optimal-logit-lower}
\end{equation}
The softmax odds identity and $\theta_t(i)\le0$ now give
\begin{align}
\varepsilon_t
&\le\frac{\varepsilon_t}{p_t}
 =\sum_{i\ne a^*}\exp\{\theta_t(i)-\theta_t(a^*)\}\nonumber\\
&\le(K-1)e^{-\theta_t(a^*)}
\le(K-1)e^{-\kappa S_t}.
\label{eq:rate-error-exponential-feedback}
\end{align}

Let $U_t:=e^{\kappa S_t}$.  For $x\in[0,1]$, convexity gives the chord bound
$e^{\kappa x}-1\le x(e^\kappa-1)$.  Thus
\[
U_{t+1}-U_t
=U_t(e^{\kappa\varepsilon_t}-1)
\le(e^\kappa-1)U_t\varepsilon_t
\le(K-1)(e^\kappa-1).
\]
Since $U_\tau=1$, summing from $\tau$ to $T-1$ yields
\begin{equation}
S_T\le\frac1\kappa
\log\!\left(1+(K-1)(e^\kappa-1)(T-\tau)\right).
\label{eq:rate-cumulative-error-log}
\end{equation}
Finally,
\[
\mu(a^*)-\sum_a\pi_t(a)\mu(a)
=\sum_{i\ne a^*}\pi_t(i)[\mu(a^*)-\mu(i)]
\le L\varepsilon_t.
\]
Summing this inequality and applying \eqref{eq:rate-cumulative-error-log} gives
\[
\sum_{t=\tau}^{T-1}
\left(\mu(a^*)-\sum_a\pi_t(a)\mu(a)\right)
\le
\frac{L}{\kappa}
\log\!\left(1+(K-1)(e^\kappa-1)(T-\tau)\right).
\]
Dividing by $T-\tau$ proves \eqref{eq:common shift-rate-explicit}.  The stated order
follows from $1+\alpha N\le(1+\alpha)(1+N)$ for $N\ge1$, with
$\alpha:=(K-1)(e^\kappa-1)$. 
\end{proof}

\adaptationnote{
The final multi-environment bandit specialization depends only on the
pooled mean vector.}

\begin{proof}[\textbf{Proof of Proposition}~\ref{prop:convergence}]

The common strict optimum implies
\[
\bar r(a^*)-\bar r(i)
=
\sum_z q_z\bigl(r_z(a^*)-r_z(i)\bigr)
>0
\qquad (i\ne a^*).
\]
By Lemma~\ref{lem:effective-common shift}, each of the three baseline
processes satisfies the assumptions of Lemma~\ref{lem:common shift} with
$C=2R$ and $\mu=\bar r$. This proves the almost-sure convergence conclusion.
Equation~\plaineqref{eq:app-pooled-drift} records their common conditional mean update.

For each baseline process, Corollary~\ref{cor:common shift-rate} gives an
almost surely finite random integer $\tau_B$. Define
\[
\delta_{\bar r}
:=
\min_{i\ne a^*}
\bigl[\bar r(a^*)-\bar r(i)\bigr]
>0,
\qquad
L_{\bar r}
:=
\max_{i\ne a^*}
\bigl[\bar r(a^*)-\bar r(i)\bigr],
\]
\[
\kappa_{\bar r}
:=
\frac{\eta\delta_{\bar r}}{4},
\qquad
\alpha_{\bar r}
:=
(K-1)\bigl(e^{\kappa_{\bar r}}-1\bigr),
\]
and
\[
\mathcal E_t^B
:=
\bar r(a^*)-\sum_a\pi_t^B(a)\bar r(a).
\]
For every integer $T>\tau_B$, putting $N:=T-\tau_B$, the corollary gives
\[
\frac{1}{N}
\sum_{t=\tau_B}^{T-1}\mathcal E_t^B
\le
\frac{L_{\bar r}}{\kappa_{\bar r}N}
\log\!\left(1+\alpha_{\bar r}N\right).
\]
The sum before $\tau_B$
\[
C_B^{\mathrm{pre}}
:=
\sum_{t=0}^{\tau_B-1}\mathcal E_t^B
\]
is finite almost surely. Hence, for every $T>\tau_B$,
\[
\frac1T\sum_{t=0}^{T-1}\mathcal E_t^B
\le
\frac{C_B^{\mathrm{pre}}}{T}
+
\frac{L_{\bar r}}{\kappa_{\bar r}T}
\log\!\left(
1+\alpha_{\bar r}(T-\tau_B)
\right).
\]
Set
\[
T_{0,B}:=\max\{\tau_B+1,2\}
\]
and
\[
C_B
:=
\frac{C_B^{\mathrm{pre}}}{\log 2}
+
\frac{L_{\bar r}}{\kappa_{\bar r}}
\left(
1+
\frac{\log(1+\alpha_{\bar r})}{\log 2}
\right).
\]
Both are finite almost surely. For every integer $T\ge T_{0,B}$,
\[
\log\!\left(
1+\alpha_{\bar r}(T-\tau_B)
\right)
\le
\log(1+\alpha_{\bar r})+\log T,
\]
and therefore
\[
\frac1T\sum_{t=0}^{T-1}\mathcal E_t^B
\le
C_B\frac{\log T}{T}.
\]
This proves~\eqref{eq:prop1-rate}. The entrance time may differ across
baseline processes, but the same deterministic bound applies after entrance as a
function of the tail length. The argument does not require the
shared or conditional offset to converge.
\end{proof}

\section{Proof of Proposition~\ref{prop:single-ratchet}}
\label{app:single-ratchet}

\begin{proof}
Let $p:=\pi(a^*)$, $\pi_i:=\pi(i)$, and
$x_i:=\theta(a^*)-\theta(i)$ for $i\ne a^*$. Finite logits and increments bounded over
a single step imply that every softmax probability at finite time is strictly positive.
Write $r_*=r(a^*)$, $r_{(2)}=\max_{i\ne a^*}r(i)$,
$r_-=\min_a r(a)$, and define the finite threshold
\[
\bar p:=\max\left\{\frac12,
\frac{r_{(2)}-r_-}{r_*-r_-}\right\}<1.
\]

\emph{Value baseline.}
Suppose $p>\bar p$. Then $p>1/2$, so $p>\pi_i$ for every rival, and
\[
V^\pi\ge p r_*+(1-p)r_->r_{(2)}.
\]
At any finite time, softmax has full support and $a^*$ is uniquely optimal, so also $V^\pi<r_*$. 
Consequently $A(a^*)=r_*-V^\pi>0$ and $A(i)=r(i)-V^\pi<0$ for every rival. If $a^*$
is sampled, then, for every $i\ne a^*$,
\begin{equation}
\Delta x_i=\eta A(a^*)\bigl(1-p+\pi_i\bigr)>0.
\label{eq:app-opt-branch}
\end{equation}
If rival $j$ is sampled, then
\begin{equation}
\Delta x_j=-\eta A(j)\bigl(1+p-\pi_j\bigr)>0,
\qquad
\Delta x_i=\eta A(j)(\pi_i-p)>0\quad(i\ne j).
\label{eq:app-rival-branch}
\end{equation}
Thus every possible branch strictly increases every margin between the optimal arm and a rival. Since
\[
\pi(a^*)=\left(1+\sum_{i\ne a^*}e^{-x_i}\right)^{-1},
\]
it strictly increases as well. The set $\{p>\bar p\}$ is therefore invariant under future updates.
Proposition~\ref{prop:convergence} gives $p_t\to1$ almost surely, so its first entrance
time $\tau_V$ is almost surely finite, proving \eqref{eq:single-ratchet}.  For $K=2$,
the two displays reduce to the two familiar positive margin increments under value centering.

\emph{No baseline.}
Define the nonempty set of rivals with positive rewards
$\mathcal I_+:=\{i\ne a^*:r(i)>0\}$.
Let $M:=\sum_{i\ne a^*}x_i$. For a sampled arm $a$, direct summation of the score update
gives the exact identity
\begin{equation}
\Delta M=K\eta r(a)\bigl(\ind{a=a^*}-p\bigr).
\label{eq:app-total-margin}
\end{equation}
Hence every draw $a=i\in\mathcal I_+$ gives
$\Delta M=-K\eta r(i)p<0$. Conditional on the history before the action, its probability is
$\sum_{i\in\mathcal I_+}\pi_t(i)>0$ at every finite time.

Again by Proposition~\ref{prop:convergence}, almost surely there is a finite time after
which $p_t>1/2$. Thereafter, if $i\in\mathcal I_+$ is sampled, then
\[
\Delta x_i=-\eta r(i)(1+p_t-\pi_t(i))<0,
\qquad
\Delta x_j=\eta r(i)(\pi_t(j)-p_t)<0\quad(j\ne i).
\]
All margins between the optimal arm and a rival, and therefore $p_t$, strictly decrease.
The exploration argument in the proof of Lemma~\ref{lem:common shift} shows that
softmax REINFORCE with bounded updates samples every arm infinitely often almost surely.
Thus, on the intersection of
the convergence and exploration events, which has probability one, every fixed $i\in\mathcal I_+$ produces
infinitely many strict drawdowns. This proves the second claim and
\eqref{eq:no-eventual-ratchet}.
\end{proof}

\section{Proofs for Section~\ref{sec:mismatch}}
\label{app:mismatchproofs}

\begin{lemma}[Population targets of shared and conditional critics]
\label{lem:projection}
Assume $G\in L^2$. Among square-integrable critics that observe $S$ but not
$Z$, the minimizer of the population squared error is
\[
\bar V^\pi(S):=\E[G\mid S]
=\sum_zq_z^\pi(S)V_z^\pi(S)
\quad\text{almost surely}.
\]
Among critics that also observe $Z$, the corresponding minimizer is
$V_Z^\pi(S):=\E[G\mid S,Z]$. Their optimal population risks differ by
\[
\E\!\left[(V_Z^\pi(S)-\bar V^\pi(S))^2\right].
\]
Values outside the support of the data distribution are unconstrained.
\end{lemma}

\begin{proof}
For any shared critic $\widehat V(S)$, conditional expectation gives
\[
\E[(\widehat V(S)-G)^2]
=\E[(\widehat V(S)-\E[G\mid S])^2]
+\E[(G-\E[G\mid S])^2].
\]
The first term is uniquely minimized almost surely by
$\widehat V(S)=\bar V^\pi(S)$. Conditioning instead on $(S,Z)$ gives the
conditional target $V_Z^\pi(S)$. Finally, applying the same orthogonal
decomposition to $V_Z^\pi(S)$ and $\bar V^\pi(S)$ gives the stated risk gap.
\end{proof}

\begin{proof}[\textbf{Proof of Proposition}~\ref{prop:pool-ratchet}]
Write
$\pi_t^s:=\pi_t^{B^{\mathrm{shared}}}$ and
$\pi_t^c:=\pi_t^{B^{\mathrm{cond}}}$, with
$p_t^B:=\pi_t^B(a^*)$. For the two processes, define
\[
A_{s,t}(z,a):=r_z(a)-\bar V^{\pi_t^s},
\qquad
A_{c,t}(z,a):=r_z(a)-V_z^{\pi_t^c},
\]
where $V_z^\pi:=\sum_a\pi(a)r_z(a)$ and
$\bar V^\pi:=\sum_zq_zV_z^\pi$. These are exactly the centered residuals induced
by the shared and conditional oracle baselines defined in Section~\ref{sec:setting}.

For the conditional process, put
\[
r_z^{(2)}:=\max_{i\ne a^*}r_z(i),\qquad
r_z^-:=\min_a r_z(a),\qquad
\bar p_c:=\max_z\max\left\{\frac12,
\frac{r_z^{(2)}-r_z^-}{r_z(a^*)-r_z^-}\right\}<1.
\]
For the shared process, put
\[
\bar r_*:=\sum_zq_zr_z(a^*),\qquad
\mathcal Z_-:=\{z:r_z(a^*)<\bar r_*\},\qquad
q_-:=\sum_{z\in\mathcal Z_-}q_z,
\]
and abbreviate
$e_z^\star:=r_z(a^*)-\bar r_*$ and
$g_z(i):=r_z(a^*)-r_z(i)$.

\emph{Conditional ratchet.}
Let $p:=p_t^c$. If $p>\bar p_c$, then, simultaneously for every environment $z$,
$p>1/2$ and
\[
V_z^{\pi_t^c}\ge p\,r_z(a^*)+(1-p)r_z^->r_z^{(2)}.
\]
Thus $A_{c,t}(z,a^*)>0>A_{c,t}(z,i)$ for every rival. The branch calculations
in \eqref{eq:app-opt-branch} and \eqref{eq:app-rival-branch}, applied with the rewards of the
sampled environment, show that every possible combination of environment and action strictly increases
every margin between the optimal arm and a rival. Hence $p$ strictly increases and the region
$\{p>\bar p_c\}$ is invariant under future updates. Proposition~\ref{prop:convergence} gives
$p_t^c\to1$ almost surely, so the first entrance time $\tau_c$ is finite almost surely.
This proves \eqref{eq:conditional-pool-ratchet}.

\emph{Limiting shared signs.}
Along the process with the shared baseline, Proposition~\ref{prop:convergence} yields
$\pi_t^s(a^*)\to1$ almost surely. Consequently,
\[
\bar V^{\pi_t^s}\longrightarrow
\bar r_*:=\sum_zq_zr_z(a^*).
\]
For each $z\in\mathcal Z_-$ define
$\delta_z:=\bar r_*-r_z(a^*)>0$. Then
\[
A_{s,t}(z,a^*)=r_z(a^*)-\bar V^{\pi_t^s}\longrightarrow-\delta_z
\qquad\text{almost surely}.
\]
More generally, for every rival $i$,
\[
A_{s,t}(z,i)=r_z(i)-\bar V^{\pi_t^s}
\longrightarrow r_z(i)-\bar r_*=e_z^\star-g_z(i).
\]
These limits give the signs of the optimal arm in hard environments and rival arms in easy environments discussed in
Section~\ref{sec:mismatch}.

\emph{Mismatch severity.}
Let $\bar r(a):=\sum_zq_zr_z(a)$ and define the suboptimality in average reward
$\mathcal E_t^s:=\bar r_*-\sum_a\pi_t^s(a)\bar r(a)
=\bar r_*-\bar V^{\pi_t^s}\ge0$.
For a hard environment $z$, the identity
$A_{s,t}(z,a^*)=\mathcal E_t^s-\delta_z$ is exact. Hence its update to the optimal arm flips
sign precisely when $\mathcal E_t^s<\delta_z$, and after the flip its residual magnitude
is $\delta_z-\mathcal E_t^s$. A larger limiting mismatch therefore moves the boundary at
which the sign flips to a larger remaining error in average reward and increases the drawdown magnitude at a
fixed policy.  Likewise, for an easy environment let
$e_z^+:=r_z(a^*)-\bar r_*>0$.  Every rival $i$ has the exact shared residual
$A_{s,t}(z,i)=e_z^+-g_z(i)+\mathcal E_t^s$; increasing $e_z^+$ lifts more rivals above
zero and strengthens every promoted branch.  These identities order onset by policy
quality $\mathcal E_t^s$: along the same trajectory under the shared baseline, a larger $\delta_z$ threshold
is crossed no later than a smaller one.  They do not, by themselves, order hitting
times across separately trained processes.  For a fixed set $\mathcal Z_-$, severity
changes onset and magnitude, not the limiting trigger frequency $q_-$ proved below;
changing the offsets can, of course, change membership in $\mathcal Z_-$.

\emph{Persistent shared drawdowns.}
Fix $z\in\mathcal Z_-$. For every $\epsilon\in(0,\delta_z)$, there is an almost surely finite time
$T_{z,\epsilon}$ such that
$A_{s,t}(z,a^*)\le-(\delta_z-\epsilon)$ for all $t\ge T_{z,\epsilon}$.
On a round with $(Z_t,a_t)=(z,a^*)$, the softmax policy gradient update gives, for every
$i\ne a^*$,
\begin{align*}
&\bigl[\theta^s_{t+1}(a^*)-\theta^s_{t+1}(i)\bigr]
-\bigl[\theta^s_t(a^*)-\theta^s_t(i)\bigr]\\
&\qquad=
\eta A_{s,t}(z,a^*)
\bigl[(1-\pi^s_t(a^*))-(-\pi^s_t(i))\bigr]\\
&\qquad=
\eta A_{s,t}(z,a^*)
\bigl(1-\pi^s_t(a^*)+\pi^s_t(i)\bigr)\\
&\qquad\le
-\eta(\delta_z-\epsilon)
\bigl(1-\pi^s_t(a^*)+\pi^s_t(i)\bigr)<0.
\end{align*}
The optimal probability can be written as
\[
\pi^s_t(a^*)=
\left(1+\sum_{i\ne a^*}
\exp\{-[\theta^s_t(a^*)-\theta^s_t(i)]\}\right)^{-1},
\]
which is strictly increasing in every margin between the optimal arm and a rival. Their simultaneous strict
decrease therefore implies $\pi^s_{t+1}(a^*)<\pi^s_t(a^*)$.

Let $E_t^z:=\{Z_t=z,a_t=a^*\}$. With $\mathcal F_t$ denoting the history before the
environment and arm are sampled (so $E_t^z\in\mathcal F_{t+1}$),
\[
\Prob(E_t^z\mid\mathcal F_t)=q_z\pi^s_t(a^*)\longrightarrow q_z>0
\qquad\text{almost surely}.
\]
Thus $\sum_t\Prob(E_t^z\mid\mathcal F_t)=\infty$ almost surely, and L\'evy's conditional
Borel--Cantelli lemma gives $\Prob(E_t^z\ \mathrm{i.o.})=1$. Every sufficiently late
occurrence is a strict drawdown.

For the frequency statement, let
$J_t:=\ind{Z_t\in\mathcal Z_-,\,a_t=a^*}$ and
$q_-:=\sum_{z\in\mathcal Z_-}q_z$. Then
\[
\E[J_t\mid\mathcal F_t]=q_-\pi^s_t(a^*)\longrightarrow q_-.
\]
The differences $J_t-\E[J_t\mid\mathcal F_t]$ are bounded martingale differences, so
the martingale strong law and Ces\`aro convergence imply
\[
\frac1n\sum_{t=0}^{n-1}J_t\longrightarrow q_-
\qquad\text{almost surely}.
\]
Because $\mathcal Z_-$ is finite, the common time
$T_-:=\max_{z\in\mathcal Z_-}T_{z,\delta_z/2}$ is almost surely finite. After $T_-$,
every event counted by $J_t$ is a strict drawdown. Conversely, if
$Z_t\notin\mathcal Z_-$ and $a_t=a^*$, then
$A_{s,t}(Z_t,a^*)=r_{Z_t}(a^*)-\bar r_*+\mathcal E_t^s>0$ at every finite time,
so that branch increases the optimal probability. Let
$I_t^\downarrow:=\ind{\pi^s_{t+1}(a^*)<\pi^s_t(a^*)}$ and
$R_t:=\ind{a_t\ne a^*}$. For every $t\ge T_-$,
\[
J_t\le I_t^\downarrow\le J_t+R_t.
\]
The same martingale strong law gives
$n^{-1}\sum_{t<n}R_t\to0$ almost surely because
$\E[R_t\mid\mathcal F_t]=1-\pi_t^s(a^*)\to0$. Therefore
\[
\frac1n\sum_{t=0}^{n-1}I_t^\downarrow\longrightarrow q_->0
\qquad\text{almost surely}.
\]
This exact frequency statement implies the infinitely often conclusion in
\eqref{eq:shared-persistent-drawdown}.
\end{proof}

\paragraph{Branch comparisons for the same sample.}
Fix an interior policy $\pi$, write $p:=\pi(a^*)$, and define each environment value
$V_z:=\sum_a\pi(a)r_z(a)$.  Write $\bar V^\pi:=\sum_zq_zV_z$ for the shared value,
and let $E$ and $H$ denote environments satisfying $V_E>\bar V^\pi>V_H$.
Condition on the history before the update and compare the conditional and shared baselines on
the same sampled tuple $(Z,a,G)$, where $G=r_Z(a)$ in the deterministic bandit. Writing
$e_Z^\pi:=V_Z-\bar V^\pi$, the score update with one logit per arm gives the following
exact identity when both baselines are evaluated at the same policy and on the same sample:
\begin{equation}
\Delta\theta^{\mathrm{shared}}-\Delta\theta^{\mathrm{cond}}
=\eta e_Z^\pi(\mathbf e_a-\pi).
\label{eq:app-update-shift}
\end{equation}
This comparison couples the two baselines for one branch; after their policies diverge,
it does not equate their subsequent updates. For
$x_i:=\theta(a^*)-\theta(i)$, an update using baseline $B$ gives
\[
\Delta x_i(B)
=\eta(G-B)\bigl(\ind{a=a^*}-\ind{a=i}-p+\pi(i)\bigr).
\]

On the event in the easy environment where $Z=E$, $a=i\ne a^*$, and
$\bar V^\pi<G<V_E$, this becomes
\begin{align*}
\Delta x_i^c
&=\eta(V_E-G)(1+p-\pi(i))>0,\\
\Delta x_i^s
&=-\eta(G-\bar V^\pi)(1+p-\pi(i))<0.
\end{align*}
Thus the same rival sample corrects the policy under conditional centering but moves it
in the wrong direction under shared centering.

On the event in the hard environment where $Z=H$, $a=a^*$, and $G>V_H$, it gives
\begin{align*}
\Delta x_i^c&=\eta(G-V_H)(1-p+\pi(i))>0,\\
\Delta x_i^s&=\eta(G-\bar V^\pi)(1-p+\pi(i)).
\end{align*}
Their difference is $\eta(\bar V^\pi-V_H)(1-p+\pi(i))$. If
$V_H<G<\bar V^\pi$, the second line is
negative, so shared centering reverses an otherwise helpful update to the optimal arm. If
$G>\bar V^\pi$, both updates have the correct sign, but their magnitude ratio is
$(G-\bar V^\pi)/(G-V_H)\in(0,1)$, so shared centering attenuates the helpful update.

\paragraph{Effect of the learning rate.}
For a fixed sampled branch, the logit discrepancy created by replacing the
conditional value with the shared value is proportional to $\eta$.  The resulting
policy trajectory is nonlinear, so this local fact does not imply a general ordering
of returns across learning rates.  Figure~\ref{fig:learning-rate-sweep} shows the
corresponding effect at the same finite budget in the instance from Figure~\ref{fig:simplex}.

\begin{center}
\begin{minipage}{\linewidth}
\centering
\includegraphics[width=\linewidth]{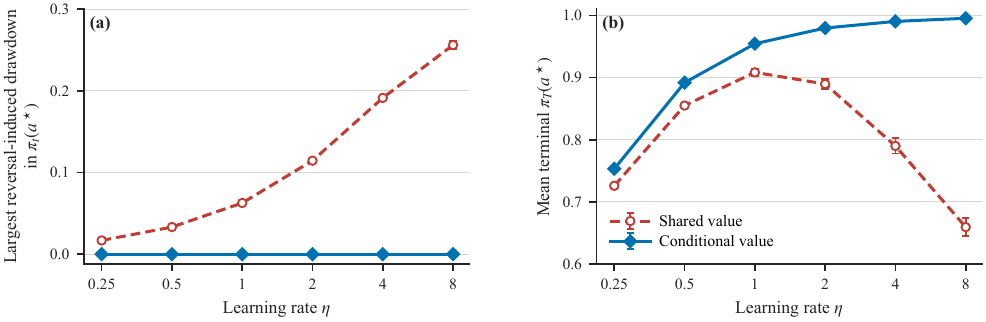}
\captionof{figure}{\textbf{At a fixed horizon, larger steps amplify the separation caused
by reversed updates in the instance from Figure~\ref{fig:simplex}.}
For each learning rate, the shared and conditional value baselines reuse the same paired
random streams. \textbf{(a)} Mean across $4{,}000$ runs of the largest decrease in a single step
in $\pi_t(a^*)$ on steps where $G-\bar V^{\pi_t}$ and
$G-V_{Z_t}^{\pi_t}$ have opposite signs; the conditional reference is zero by construction.
\textbf{(b)} Mean terminal probability of the optimal arm at $T=100$.
Error bars are $95\%$ confidence intervals. In this fixed instance, larger $\eta$
magnifies reversed branches and widens the separation at this fixed budget; the plot is not a
general monotonicity claim about learning rate or return.}
\label{fig:learning-rate-sweep}
\end{minipage}
\end{center}


\begin{lemma}[Attenuation in a hard environment before a sign flip]
\label{lem:hard-attenuation}
Fix a deterministic bandit with $K\ge2$ arms, unique optimal arm $a^*$, gaps
$\Delta_i:=r(a^*)-r(i)>0$, and learning rate $\eta>0$.  There exists
$\varepsilon_0>0$, depending only on $\eta$ and the reward gaps, with the following
property.  Write an interior frozen policy as
$p_*:=\pi(a^*)=1-\varepsilon$ and
$\pi(i)=\varepsilon w_i$, where every sum over $i$ below ranges over $i\ne a^*$ and
$\sum_iw_i=1$.  If
$0<\varepsilon<\varepsilon_0$, let $V:=\sum_a\pi(a)r(a)$ and compare the conditional
bar $V$ with a shared bar $V+d$ in a hard environment, where
$0<d<r(a^*)-V$.  Every sampled branch still moves probability toward $a^*$ under
both bars, and their expected logit updates are identical.  Let
$\pi_B^+(a^*)$ denote the probability of the optimal arm after one update using baseline $B$
and a sampled arm $a\sim\pi$.  Then
\[
\E_{a\sim\pi}\!\left[\pi_{V+d}^+(a^*)\right]
<
\E_{a\sim\pi}\!\left[\pi_V^+(a^*)\right].
\]
\end{lemma}

\begin{proof}
Put $\bar\Delta_w:=\sum_iw_i\Delta_i$.  The conditional advantages are
$A_*:=r(a^*)-V=\varepsilon\bar\Delta_w$ and
$A_i:=r(i)-V=-\Delta_i+\varepsilon\bar\Delta_w$.  Thus, after reducing
$\varepsilon_0$ if necessary, $A_*>0>A_i$ for every rival and
$p_*>1/2>\pi_i$.  Since $0<d<A_*$, raising the bar by $d$ preserves all signs;
the corresponding score directions strictly increase every optimal--rival margin.
Moreover, on a sampled arm $a$ the
difference between the two logit updates is $-\eta d(\mathbf e_a-\pi)$, whose expectation is
zero because $\sum_a\pi(a)(\mathbf e_a-\pi)=0$.

It remains to compare the nonlinear probabilities.  Interpolate the bar as
$V+\lambda$, $0\le\lambda\le d$, and let $H(\lambda)$ be the expected
optimal probability after the update. On the optimal branch use direction
$u_*=\mathbf e_{a^*}-\pi$; on a
branch for rival $i$ use $u_i=\pi-\mathbf e_i$. If
$F_a(t):=[\softmax(\theta+t u_a)]_{a^*}$, then
\[
H(\lambda)
=p_*F_*\!\left(\eta(A_*-\lambda)\right)
 +\sum_i\varepsilon w_iF_i\!\left(\eta(-A_i+\lambda)\right).
\]
The exact odds between the optimal arm and each rival give, uniformly over the vector $w$ of rival masses and
$\lambda\in[0,A_*]$,
\begin{align*}
p_*F_*'\!\left(\eta(A_*-\lambda)\right)
 &=\varepsilon^2\left(1+\sum_iw_i^2\right)+O(\varepsilon^3),\\
\sum_i\varepsilon w_iF_i'\!\left(\eta(-A_i+\lambda)\right)
 &=\varepsilon^2\sum_iw_i\!\left[
 2w_i e^{-2\eta\Delta_i}+(1-w_i)e^{-\eta\Delta_i}
 \right]+O(\varepsilon^3).
\end{align*}
For completeness, these expansions follow by dividing the probability of each arm after the update by that of the optimal arm. On the optimal branch the rival-$i$ odds are
\[
\frac{\varepsilon w_i}{1-\varepsilon}
\exp\{-t\varepsilon(1+w_i)\},
\]
where $t=O(\varepsilon)$; on rival-$i$'s branch, its own odds acquire
$e^{-2\eta\Delta_i}+O(\varepsilon)$ and every other rival's odds acquire
$e^{-\eta\Delta_i}+O(\varepsilon)$.

Differentiating $H$ now yields
\[
H'(\lambda)=\eta[-\varepsilon^2\Gamma(w)+R_\varepsilon(w,\lambda)],
\qquad R_\varepsilon(w,\lambda)=O(\varepsilon^3),
\]
where
\[
\Gamma(w)=\sum_i\left[
2w_i^2(1-e^{-2\eta\Delta_i})
+w_i(1-w_i)(1-e^{-\eta\Delta_i})
\right].
\]
With $\Delta_{\min}:=\min_i\Delta_i$,
$\Gamma(w)\ge g_0:=1-e^{-\eta\Delta_{\min}}>0$, uniformly in $w$.
To make the remainder uniform explicit, put $\ell:=\lambda/\varepsilon$; then
$0\le\ell\le\bar\Delta_w\le\Delta_{\max}$.  The odds expressions above extend to
analytic functions of $(\varepsilon,w,\ell)$ on a compact set, with denominators
uniformly bounded away from zero.  After factoring out $\varepsilon^2$, the remaining
coefficients are uniformly continuously differentiable there.  A uniform first-order
Taylor bound therefore gives constants $C_0\in(0,\infty)$ and $\varepsilon_1>0$, depending only on $K$, $\eta$, and the
gaps, such that $|R_\varepsilon(w,\lambda)|\le C_0\varepsilon^3$ for all $w$,
$0\le\lambda\le A_*$, and
$0<\varepsilon\le\varepsilon_1$.  Taking, for example,
\[
\varepsilon_0\le
\min\!\left\{\varepsilon_1,\frac12,
\frac{\Delta_{\min}}{2\Delta_{\max}},
\frac{g_0}{2C_0}\right\},
\qquad \Delta_{\max}:=\max_i\Delta_i,
\]
makes every rival advantage negative and gives
$H'(\lambda)\le-(\eta g_0/2)\varepsilon^2<0$ throughout $[0,d]$.  Integrating gives
$H(d)<H(0)$.  When rival mass is small, the gain from strengthening all rare
branches with negative reinforcement is therefore smaller than the loss from weakening the
frequent optimal branch, even before any sign reversal.
\end{proof}

\paragraph{Large $\varepsilon$ in hard environments.}
The attenuation mechanism in Lemma~\ref{lem:hard-attenuation} need not be confined to a
nearly deterministic policy.
Consider the following instance with three arms and two environments:
\[
q_H=\frac1{20},
\quad r_H=(1,0,0),
\qquad
q_E=\frac{19}{20},
\quad r_E=(1,\tfrac{19}{20},\tfrac{19}{20}),
\]
with $\eta=3$ and
$\pi_0=(1/2,1/4,1/4)$, so the initial rival mass is already
$\varepsilon=1/2$.  In the hard environment, conditional centering gives
optimal/rival residuals $(0.5,-0.5)$, whereas sharing gives
$(0.04875,-0.95125)$.  Sharing therefore buys stronger negative reinforcement by
removing most of the much larger positive update to the optimal arm.

This comparison does not rely on a sign reversal: whenever $p=\pi(a^*)\ge1/2$, the
shared optimal residual remains positive and every rival residual remains negative;
each branch preserves this condition.  Figure~\ref{fig:symmetric-hard-attenuation}
shows that the conditional curve remains visibly above the shared curve over the
displayed horizon. Thus the attenuation mechanism can hold when $\varepsilon$ is bounded
away from zero, not only arbitrarily close to the optimum.

\begin{center}
\begin{minipage}{\linewidth}
  \centering
  \includegraphics[width=0.69\textwidth]{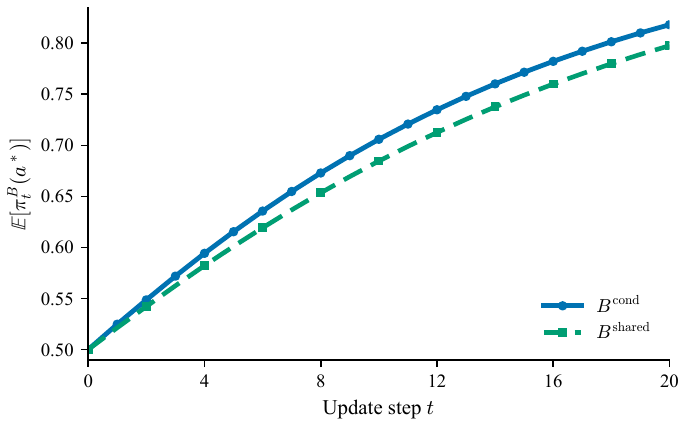}
  \captionof{figure}{Separation after finitely many steps without sign reversal. Algorithm~\ref{alg:eispg} with
  $K=3$, $q_H=1/20$, $q_E=19/20$, $r_H=(1,0,0)$,
  $r_E=(1,19/20,19/20)$, $\eta=3$, and $\pi_0=(1/2,1/4,1/4)$.
  Curves are Monte Carlo means of $\pi_t^B(a^*)$ over $400{,}000$ paired trajectories.}
  \label{fig:symmetric-hard-attenuation}
\end{minipage}
\end{center}

\paragraph{Derivation and closed form of the GAE mismatch recursion.}
Generalized Advantage Estimation (GAE) is a TD($\lambda$) estimator built from
bootstrapped temporal difference residuals \citep{schulman2016gae}. Consider one rollout
segment $(s_u,r_u,s_{u+1})_{u=t}^{T-1}$ from a fixed environment
$z$, and put $n:=T-t$.  Let $m_u\in\{0,1\}$ be the continuation mask after transition
$u$: it is zero when that transition terminates the episode and one otherwise. Along this
rollout, write
\[
V_u^{\mathrm{shared}}:=V^{\mathrm{shared}}(s_u),
\qquad
V_u^{\mathrm{cond}}:=V^{\mathrm{cond}}(s_u,z),
\qquad
e_u:=V_u^{\mathrm{cond}}-V_u^{\mathrm{shared}}.
\]
For each $u\in\{t,\ldots,T-1\}$ and
$B\in\{\mathrm{shared},\mathrm{cond}\}$, define
\[
\delta_u^B=r_u+\gamma m_uV_{u+1}^B-V_u^B,
\qquad
\widehat A_u^{B,\lambda}
=\sum_{\ell=0}^{T-u-1}(\gamma\lambda)^\ell M_{u,\ell}\delta_{u+\ell}^B,
\]
where $M_{u,0}:=1$ and $M_{u,k}:=\prod_{j=0}^{k-1}m_{u+j}$ for $k\ge1$. Then
\[
\delta_u^{\mathrm{shared}}-\delta_u^{\mathrm{cond}}
=e_u-\gamma m_ue_{u+1}.
\]
Define the GAE mismatch by
$D_u^\lambda:=\widehat A_u^{\mathrm{shared},\lambda}
-\widehat A_u^{\mathrm{cond},\lambda}$ for
$u<T$, with $D_T^\lambda:=0$ at the rollout boundary.  The standard backward
GAE recursion then gives, for $u=t,\ldots,T-1$,
\begin{equation}
D_u^\lambda=e_u-\gamma m_ue_{u+1}
+\gamma\lambda m_uD_{u+1}^\lambda,
\qquad D_T^\lambda=0.
\label{eq:gae-mismatch-filter}
\end{equation}
Unrolling this recursion, or equivalently
substituting into the finite sum and reindexing its second term, gives
\begin{align*}
D_t^\lambda
={}\sum_{\ell=0}^{n-1}(\gamma\lambda)^\ell M_{t,\ell}e_{t+\ell} 
-\sum_{k=1}^{n}\gamma(\gamma\lambda)^{k-1}M_{t,k}e_{t+k}.
\end{align*}
Collecting the coefficient of each $e_{t+k}$ yields the closed form
\begin{align}
D_t^\lambda={}e_t-\gamma(1-\lambda)\sum_{k=1}^{n-1}(\gamma\lambda)^{k-1}
M_{t,k}e_{t+k}\nonumber 
-\gamma(\gamma\lambda)^{n-1}M_{t,n}e_{t+n}.
\label{eq:gae-mismatch-closed}
\end{align}
For a single TD step ($\lambda=0$), the recursion reduces to
$D_t^0=e_t-\gamma m_te_{t+1}$.  For $\lambda=1$, it telescopes to
$D_t^1=e_t-\gamma^nM_{t,n}e_T$.  At a genuine terminal endpoint,
$m_{T-1}=0$ and hence $M_{t,n}=0$; at a bootstrapped nonterminal truncation,
$m_{T-1}=1$ and the final mismatch term remains.  The identity compares the same
rollout, rewards, evaluation states, and masks before batchwise advantage normalization
and before entering PPO's clipped objective.
It shows that GAE linearly filters critic mismatch and therefore does not erase it in
general, although particular offset sequences can cancel.  The environment index is
fixed along the trajectory; if values vary over time, the stage is included in the
state label.


\section{Full experimental details}
\label{app:expdetails}


\subsection{Conditioning architectures: FiLM, multihead, and PopArt}
\label{app:archdetails}

Every main comparison changes only the value function; the diagnostic in
Figure~\ref{fig:bipedal-conditioning-location} that also conditions the actor is the sole exception. On the MLP
benchmarks (CartPole, the MuJoCo suite, BipedalWalker) all trunks have two hidden layers
of width $64$ with $\tanh$. CartPole shares one trunk between actor and critic, and its
reported conditional critic uses a multihead value readout. The continuous control
environments use separate actor and critic trunks, and FiLM modulates the critic trunk
only. Continuous control actors output a Gaussian mean from the actor trunk with a
learned $\log\sigma$ independent of state (initialized to $0$); hidden layers use
orthogonal initialization, with gain $0.01$ on the policy mean head and $1.0$ on the
value head. Writing $h(s)\in\R^{64}$ for the critic trunk features, the three MLP
critics are
\begin{align*}
V_{\mathrm{sh}}(s)&=w^\top h(s)+b,\\
V_{\mathrm{FiLM}}(s,z)&=w^\top\bigl(h(s)\odot(1+\gamma_z)+\beta_z\bigr)+b,
\qquad \gamma_z=W_\gamma c_z+u_\gamma,\ \ \beta_z=W_\beta c_z+u_\beta,\\
V_{\mathrm{mh}}(s,z)&=w_z^\top h(s)+b_z,
\end{align*}
where $c_z\in\R^{d}$ is a learned embedding for each level. The reference count in
Table~\ref{tab:paramcounts} uses the minimal choice $d=1$ for CartPole; the evaluated
FiLM critics use $d=4$ on the MuJoCo bodies and $d=16$ on BipedalWalker, while the
Procgen implementation below also uses $d=16$. The matrices
$W_\gamma,W_\beta\colon\R^{d}\to\R^{64}$ with biases $u_\gamma,u_\beta$ are shared generators
\citep{perez2018film}. The multihead critic replaces the scalar readout by one
readout row per level --- the multitask value function of
\citet{hessel2019popartmultitask} with levels playing the role of tasks. The PopArt variant
\citep{vanhasselt2016popart,hessel2019popartmultitask} augments the multihead critic
with adaptive normalization of the value targets for each head: running first and second
moments $(\mu_z,\nu_z)$ track head $z$'s target distribution, the head predicts a
normalized value, $V=\sigma_z\cdot\mathrm{out}+\mu_z$, the value loss is computed in
normalized space, and row $z$ is rescaled whenever the statistics change so the
denormalized output is preserved. The statistics are buffers, not parameters, so the
parameter count equals that of the multihead critic.

\paragraph{Identical initialization.} Every conditioned critic is initialized to
coincide \emph{exactly} with the shared critic at step $0$: $W_\gamma,W_\beta,
u_\gamma,u_\beta$ are initialized to zero, so $\gamma_z=\beta_z=0$ and $V_{\mathrm{FiLM}}(s,z)=
V_{\mathrm{sh}}(s)$; every row $(w_z,b_z)$ of the multihead readout is initialized to
the same values, so all levels return the same value. Any difference between variants is
therefore produced by learning, not by a different initialization. The CartPole scalar bias
control likewise starts from $\beta_0=\beta_1=0$.

\paragraph{The actor does not see $z$ in the proposed intervention.} In every main
comparison, $z$ is an arbitrary identity label rather than a vector of physical or
procedural environment parameters. It is a privileged signal available only to the critic
during training; the actor is architecturally identical to the baseline's, and at deployment
only the actor runs, so conditioning the critic costs nothing at test time. FiLM learns a
lookup embedding for this label and
multihead uses it only to select a value head; neither receives a structured descriptor
from which environment dynamics could be inferred directly.

\paragraph{Procgen critics.} On Procgen all variants share Procgen's large CNN trunk in the IMPALA style
\citep{espeholt2018impala,cobbe2020procgen}
with a final embedding $\phi(s)$ of dimension $256$ and a categorical policy head over
the $15$ actions; the policy head always consumes the unmodulated $\phi(s)$. The value
heads are
\begin{align*}
V_{\mathrm{sh}}(s)&=w^\top\phi(s)+b, && w\in\R^{256},\\
V_{\mathrm{FiLM}}(s,z)&=w^\top\bigl(\phi(s)+\gamma_z\odot\phi(s)+\beta_z\bigr)+b,
&& c_z\in\R^{16},\\
V_{\mathrm{mh}}(s,z)&=w_z^\top\phi(s)+b_z, && W\in\R^{200\times256}.
\end{align*}
Here $\gamma_z=W_\gamma c_z+u_\gamma$ and $\beta_z=W_\beta c_z+u_\beta$.
FiLM's cost for each level is the $16$ embedding entries; $W_\gamma,W_\beta\colon
\R^{16}\to\R^{256}$ and their biases are initialized to zero as above. The rows of the multihead critic are
initialized as copies of the shared head; its cost for each level is $257$ free parameters.

\mbox{}\vspace{-\baselineskip}\par
\paragraph{FiLM versus multihead.} The two architectures parameterize the same object
--- the correction $e_z^\pi(s)$ for each level identified in Section~\ref{sec:mismatch} --- through
opposite statistical tradeoffs, and this is the point of running both. Each multihead
readout row is an independent vector fit from \emph{only that level's data}, whereas
FiLM routes an embedding with low dimension through a generator and readout fit from
\emph{all levels jointly}. For $L$ levels and hidden width $H$, FiLM adds
$Ld+2H(d+1)$ parameters, while multihead replaces one value head with $H+1$ parameters
by $L$ such heads and therefore adds $(L-1)(H+1)$ parameters relative to the
shared model. With $10$ levels multihead is the
cheaper of the two; with the $100$ levels of BipedalWalker it becomes the more
expensive one ($+54.3\%$ vs.\ $+31.9\%$).
Table~\ref{tab:paramcounts} gives the corresponding counts for each architecture; the
coverage paragraph below states which combinations were evaluated.

\begin{center}
\begin{minipage}{\linewidth}
\centering
\captionof{table}{Total trainable actor and critic parameter counts for the reference
combinations of architecture and benchmark. Parentheses report increases relative to the shared
model. FiLM retains the shared value head and adds a level embedding table and two affine
generators; multihead replaces the shared value head by one head per level. The CartPole
FiLM reference uses the minimal embedding dimension $d=1$, and the Procgen counts use the
large PLR CNN \citep{jiang2021plr}. PopArt has the same trainable count as multihead.
The table includes parameterizations not used in the main comparisons; the coverage
paragraph below states which variants were evaluated.}
\label{tab:paramcounts}
\small
\setlength{\tabcolsep}{3.5pt}
\begin{tabular}{lccrrr}
\toprule
Environment & obs / act & levels & shared & FiLM & multihead\\
\midrule
CartPole      & 4 / 2  & 2   & 4{,}675  & 4{,}933 ($+5.5\%$) & 4{,}740 ($+1.4\%$)\\
BipedalWalker & 24 / 4 & 100 & 11{,}849 & 15{,}625 ($+31.9\%$) & 18{,}284 ($+54.3\%$)\\
Walker2d      & 17 / 6 & 10  & 11{,}085 & 11{,}765 ($+6.1\%$) & 11{,}670 ($+5.3\%$)\\
Hopper        & 11 / 3 & 10  & 10{,}119 & 10{,}799 ($+6.7\%$) & 10{,}704 ($+5.8\%$)\\
HalfCheetah   & 17 / 6 & 10  & 11{,}085 & 11{,}765 ($+6.1\%$) & 11{,}670 ($+5.3\%$)\\
Procgen       & $64^2{\times}3$ / 15 & 200 & 626{,}256 & 638{,}160 ($+1.9\%$) & 677{,}399 ($+8.2\%$)\\
\bottomrule
\end{tabular}
\end{minipage}
\end{center}

\paragraph{PopArt configurations.} On the MLP benchmarks the moment EMA rate is
$\beta=3\times10^{-3}$. On Procgen, where each estimate for a level sees roughly $1/200$
of the batch, the statistics update once per minibatch with rate $\beta=3\times10^{-4}$
and the standard debiasing correction, with three stabilizers: a variance floor
$\sigma_{\min}=10^{-2}$ (PopArt with a single head uses $10^{-4}$); a head's statistics update
only when it has at least $8$ samples in the minibatch; and normalization stays off
(statistics still accumulating) for the first $50$ updates --- without this warmup the
first update produced gradient spikes of order $10^{4}$ from early
returns.

\paragraph{Coverage.} The experiments actually reported are as follows. The CartPole
mechanism study reports a shared critic, a multihead critic, a multihead critic given a
constant index, and a critic with scalar biases. The main MuJoCo return comparison reports
shared and FiLM critics; the value loss diagnostic in
Figure~\ref{fig:value-loss-continuous} additionally reports multihead. BipedalWalker
additionally reports multihead, and Procgen reports shared, FiLM, multihead, and
multihead$+$PopArt.

\subsection{CartPole}
\label{app:cartpole}

CartPole \citep{barto1983neuronlike} is the classic cart--pole balancing task: a
$4$-dimensional observation, two discrete actions, reward $+1$ per step while the pole
stays up. The mechanism study retains two distinct logged level identities in both
settings. In the identical control, $(g_0,g_1)=(10,10)$; in the heterogeneous setting,
$(g_0,g_1)=(10,50)$. Every other parameter (force magnitude, pole mass, pole length, cart
mass) remains at its default. Gravity is not part of the observation, so the two levels
cannot be distinguished from the current state alone. Episodes end when the pole falls
(maximum $200$ steps), so episode lengths can differ between levels.

Train and evaluation use the \emph{same} two levels: the question is the learning signal,
not generalization to unseen levels. Returns are averaged within each level before averaging
the two levels with equal weight.

\paragraph{Mechanism variants.}
The shared critic is $V_{\mathrm{sh}}(s)=w^\top h(s)+b$, and the multihead critic is
$V_{\mathrm{mh}}(s,z)=w_z^\top h(s)+b_z$. The constant index control uses the same
multihead parameterization but routes every sample to head $0$. The scalar bias control uses
\[
V_{\mathrm{bias}}(s,z)=V_{\mathrm{sh}}(s)+\beta_z,
\qquad \beta_0=\beta_1=0\ \text{ at initialization}.
\]
Thus the two multihead variants have $4{,}740$ trainable parameters, the shared critic has
$4{,}675$, and the critic with a scalar bias has $4{,}677$. The actor, shared trunk, optimizer,
data, and PPO pipeline are otherwise unchanged.

\paragraph{Diagnostics.}
At every logged update and for each level, value loss is the mean squared difference between
the GAE return target and the value prediction. Advantage is the mean raw GAE estimate over
transitions from that level. We compute these
levelwise quantities before PPO clipping. Figure~\ref{fig:cartpole-mechanism} averages each metric across $20$ seeds,
with shading denoting one standard error. Table~\ref{tab:hp-cartpole} lists the full
configuration.

\begin{table}[h]
\centering
\caption{CartPole mechanism study hyperparameters ($20$ seeds per variant and setting).}
\label{tab:hp-cartpole}
\small
\begin{tabular}{ll}
\toprule
Parameter & Value\\
\midrule
Algorithm & PPO \citep{schulman2017ppo}\\
Parallel environments & 16\\
Rollout length & 128\\
Total environment steps & 2{,}048{,}000 (1{,}000 updates)\\
PPO epochs / minibatches & 4 / 4\\
Clip parameter & 0.2\\
Learning rate & $3\times10^{-4}$ (Adam, $\epsilon=10^{-5}$)\\
Discount $\gamma$ / GAE $\lambda$ & 0.99 / 0.95 \citep{schulman2016gae}\\
Value loss coefficient & 0.5\\
Entropy coefficient & 0.01\\
Max gradient norm & 0.5\\
Advantage normalization & off (all variants)\\
Hidden width & 64\\
Conditional readout & multihead with two value heads\\
Gravity pairs & $(10,10)$ and $(10,50)$\\
Seeds & 1--20\\
\bottomrule
\end{tabular}
\end{table}

\subsection{MuJoCo suite: Walker2d, Hopper, HalfCheetah}
\label{app:mujoco}

The three standard MuJoCo locomotion bodies \citep{todorov2012mujoco} (Gym v4;
\citealp{brockman2016gym}) use an
identical protocol so they can be compared directly:

\begin{center}
\small
\begin{tabular}{lccc}
\toprule
 & Walker2d-v4 & Hopper-v4 & HalfCheetah-v4\\
\midrule
Observation dim & 17 & 11 & 17\\
Action dim & 6 & 3 & 6\\
Episode length & 1000 (fixed) & 1000 (fixed) & 1000 (fixed)\\
Mass multiplier & $[0.5,2.5]$ & $[0.5,2.5]$ & $[0.5,2.5]$\\
\bottomrule
\end{tabular}
\end{center}

\paragraph{Levels.} Ten levels differ in \emph{body mass only}: every
\texttt{body\_mass} entry of the model is multiplied by a scale factor for each level. The
factors are drawn from a distribution uniform in log space on $[0.5,2.5]$ and sorted, using the \emph{run seed}
as the RNG seed. Consequently the shared and conditional runs of a given seed face the
identical collection (the comparison is properly paired), while different seeds see different
collections (the result is not an artifact of one particular grid). Level $0$ is always the
lightest and level $9$ the heaviest; the mass is not explicitly provided in the observation.


\paragraph{Train / test.} Reported numbers are training returns on the ten training
environments, averaged within each level over the last $25\%$ of training and then across
levels with equal weight. The advantage heatmaps for individual levels in
Figure~\ref{fig:mujoco} plot the mean sampled GAE advantage for each level, with rows ordered from lightest to heaviest level.


\begin{table}[h]
\centering
\caption{MuJoCo suite hyperparameters ($30$ seeds per method and body).}
\label{tab:hp-mujoco}
\small
\begin{tabular}{ll}
\toprule
Parameter & Value\\
\midrule
Algorithm & PPO\\
Levels & 10\\
Environments per level & 2 (20 parallel)\\
Rollout length & 256 (5{,}120 transitions per update)\\
Total environment steps & 10{,}000{,}000\\
PPO epochs / minibatches & 10 / 8\\
Clip parameter & 0.2\\
Learning rate & $3\times10^{-4}$ (Adam, $\epsilon=10^{-5}$)\\
Discount $\gamma$ / GAE $\lambda$ & 0.99 / 0.95\\
Value loss coefficient & 0.5\\
Entropy coefficient & 0.0\\
Max gradient norm & 0.5\\
Advantage normalization & on\\
Hidden width & 64\\
Level embedding dimension & 4\\
Seeds & 1--30\\
\bottomrule
\end{tabular}
\end{table}

\subsection{BipedalWalker}
\label{app:bipedal}

BipedalWalker \citep{brockman2016gym} is planar bipedal locomotion over procedurally
generated terrain: a $24$-dimensional lidar/proprioception observation and $4$
continuous torque actions. The training collection contains $100$ \emph{pinned} terrains:
$10$ \texttt{BipedalWalker-v3} terrains (levels $0$--$9$) and $90$
\texttt{BipedalWalkerHardcore-v3} terrains (levels $10$--$99$) --- a collection dominated
by hard terrains with stumps, pits, and stairs. Terrain is generated from the environment RNG at
reset, so each environment is reseeded with its own fixed terrain seed before
\emph{every} reset and replays the same terrain for the whole run. Terrain identity is
not explicitly included in the observation vector.

\paragraph{Train / test.} Training uses the $100$ pinned terrains, one parallel
environment each. Testing uses $100$ \emph{unseen} terrains generated with seed offset
$500$ --- a fresh set with the same $10/90$ normal/hardcore split that never appears in
training --- evaluated every 1M steps, one episode per terrain, reporting average returns.


\begin{table}[h]
\centering
\caption{BipedalWalker hyperparameters ($10$ seeds per method).}
\label{tab:hp-bipedal}
\small
\begin{tabular}{ll}
\toprule
Parameter & Value\\
\midrule
Algorithm & PPO\\
Parallel environments & 100 (one per terrain)\\
Rollout length & 256\\
Total environment steps & 60{,}000{,}000\\
PPO epochs / minibatches & 10 / 8\\
Clip parameter & 0.2\\
Learning rate & $3\times10^{-4}$ (Adam, $\epsilon=10^{-5}$)\\
Discount $\gamma$ / GAE $\lambda$ & 0.99 / 0.95\\
Value loss coefficient & 0.5\\
Entropy coefficient & 0.0\\
Max gradient norm & 0.5\\
Advantage normalization & on\\
Hidden width & 64\\
Level embedding dimension & 16\\
Evaluation interval & 1{,}000{,}000 steps\\
Seeds & 1--10\\
\bottomrule
\end{tabular}
\end{table}

\paragraph{Capacity control.} Figure~\ref{fig:bipedal}(c) repeats the full comparison
with each value network widened while its conditioning mechanism is held fixed, giving
every critic variant roughly five times as many trainable parameters; all other settings are
identical. The shared critic plateaus at the same level, and the conditioned critics
again reach returns of roughly $150$ and above --- the gap is not a capacity artifact.

\paragraph{Other conditioning choices.}
Figure~\ref{fig:bipedal-conditioning-location} compares two less successful ways to use
the environment index. Assigning a separate value network to every terrain removes the
shared value representation and learns substantially more slowly than FiLM and
multihead. Supplying the index to the actor as well as the multihead critic instead
collapses both training return and return on unseen terrains. Because unseen terrains
carry no valid index, we evaluate this variant that conditions the actor by running the actor
with each training index in turn and averaging the resulting returns. The successful intervention on this benchmark
is therefore deliberately asymmetric: the actor and critic representation remain shared,
while only value prediction is conditioned on the environment.

\begin{center}
\begin{minipage}{\linewidth}
  \centering
  \begin{minipage}[t]{0.49\linewidth}
    \centering
    \includegraphics[width=\linewidth]{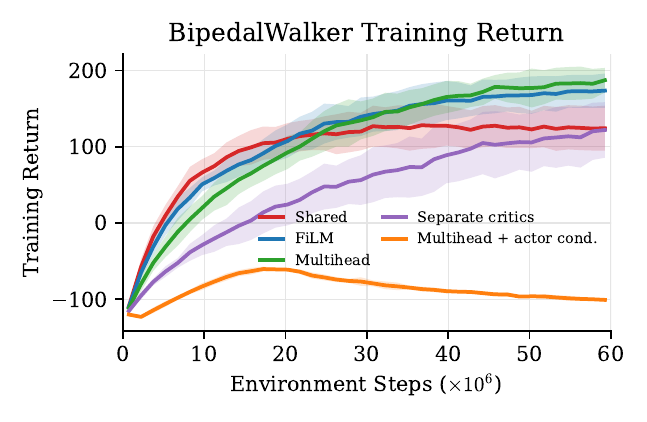}
  \end{minipage}\hfill
  \begin{minipage}[t]{0.49\linewidth}
    \centering
    \includegraphics[width=\linewidth]{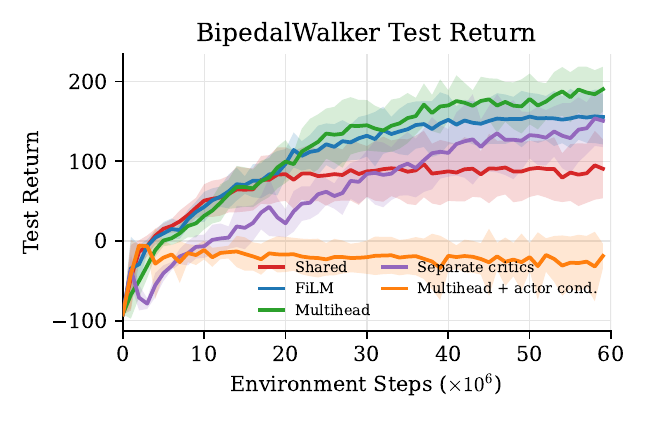}
  \end{minipage}
  \captionof{figure}{Alternative conditioning choices on BipedalWalker. Training return on the
  $100$ pinned terrains (left) and return on $100$ unseen terrains (right),
  mean$\,\pm\,$1 s.d.\ over $10$ seeds. Separate critics learn more slowly but approach
  FiLM's final test mean; also conditioning the actor causes both returns to collapse in
  this setting.}
  \label{fig:bipedal-conditioning-location}
\end{minipage}
\end{center}

\subsection{Procgen}
\label{app:procgen}

Procgen \citep{cobbe2020procgen} is a suite of $16$ procedurally generated arcade
games with $64\times64\times3$ observations and a common $15$-action discrete space; a
level seed controls the layout, assets, entity locations and spawn times, and other
details specific to each game. Each game is trained
separately. Instead of sampling levels freely, every run trains on a fixed collection of
$L=200$ \emph{pinned} levels under a hybrid difficulty split: parallel environment $i$
is permanently assigned level seed $i$, with environments $0$--$99$ drawn from the
\texttt{easy} distribution and $100$--$199$ from \texttt{hard}. The assignment never
changes during training, so level identity is exactly the environment index and
statistics for each level can be logged at every update. Levels are visited with uniform
weighting (one parallel copy each); we refer to this as the uniform level sampling reference.
Returns are normalized with the standard scaling based on the running return (discount $0.999$);
observations are only scaled by $1/255$. Advantages are computed with GAE and
normalized per update batch before the policy loss. Figure~\ref{fig:procgen-environments}
gives a visual overview of the full benchmark.

\begin{center}
\begin{minipage}{\linewidth}
\centering
\includegraphics[width=0.68\textwidth]{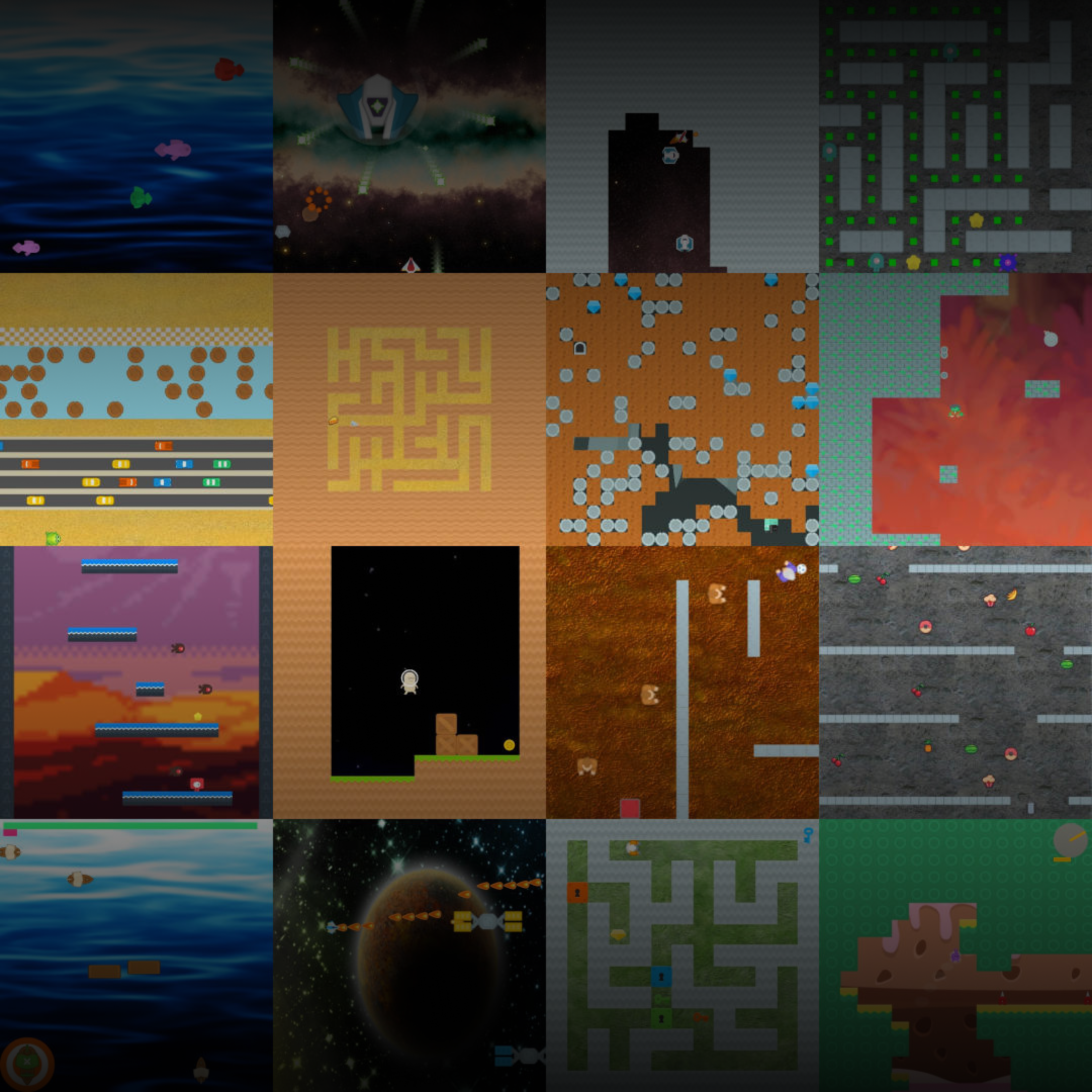}
\captionof{figure}{Example observations from all $16$ Procgen games: Bigfish, Bossfight,
Caveflyer, Chaser, Climber, Coinrun, Dodgeball, Fruitbot, Heist, Jumper, Leaper,
Maze, Miner, Ninja, Plunder, and Starpilot. Each game generates many visually and
structurally distinct levels while retaining its own mechanics. Screenshots are adapted
from the official Procgen release \citep{cobbe2020procgen}.}
\label{fig:procgen-environments}
\end{minipage}
\end{center}

\begin{table}[h]
\centering
\caption{Procgen PPO hyperparameters (per game and method).}
\label{tab:hp-procgen}
\small
\begin{tabular}{@{}lp{0.55\linewidth}@{}}
\toprule
Parameter & Value\\
\midrule
Parallel environments & 200 (one per pinned level)\\
Rollout length & 256 steps per environment\\
Batch size per update & 51{,}200\\
PPO epochs / minibatches & 3 / 8\\
Clip parameter & 0.2\\
Learning rate & $5\times10^{-4}$ (Adam, $\epsilon=10^{-5}$)\\
Discount $\gamma$ / GAE $\lambda$ & 0.999 / 0.95\\
Entropy coefficient & 0.01\\
Value loss coefficient & 0.5\\
Max gradient norm & 0.5\\
Total environment steps & 25M ($\approx488$ updates)\\
Trunk & Large CNN in the IMPALA style, $\phi(s)\in\R^{256}$\\
Level embedding dimension (FiLM) & 16\\
Advantage normalization & on (per update batch)\\
Evaluation interval & $5\times10^{5}$ steps\\
Seeds per method & 10\\
\bottomrule
\end{tabular}
\end{table}

\paragraph{Evaluation protocol.} Every $5\times10^{5}$ environment steps, a separate,
freshly reset evaluation copy of all $200$ pinned levels is rolled out with the
deterministic argmax policy for exactly one episode per level; the reported evaluation
return is the unweighted mean over the $200$ episode returns (easy/hard splits average
levels $0$--$99$ and $100$--$199$). Because every level contributes exactly one
episode, this metric weights levels equally and is unaffected by bias caused by episode length
in running averages collected under the policy. A run's final evaluation return on pinned levels is the mean of its last
$4$ evaluations. We use \emph{evaluation return on the pinned training levels} for this
metric; the main Procgen learning curves, Appendix Figure~\ref{fig:procgen-full}, and
Table~\ref{tab:popart} use it. The training curve reports the running mean
return of the last $100$ finished episodes under the stochastic behavior policy.
The rows for individual games in Table~\ref{tab:popart} report raw, unscaled episodic scores.
Normalized evaluation returns in its last row divide each run's final evaluation
return on pinned levels
by the mean final evaluation return of the shared critic under uniform level sampling over
all its runs, in percent. The last row aggregates the resulting $16\times10$ scores
for each method and reports their mean~$\pm$~1 s.d.

\paragraph{Evaluation on unseen levels.} At the final checkpoint, we evaluate the deterministic
policy for one episode on each of $600$ unseen levels per game: seeds $1000$--$1299$
under both the easy and hard modes. Levels receive equal weight. Table~\ref{tab:heldout}
reports mean$\,\pm\,$1 s.d.\ over the same $10$ training seeds; its normalized row divides
each score for a game and seed by that game's shared critic mean and aggregates the resulting
$16\times10$ normalized scores.

\begin{table}[t]
\centering\small
\caption{Procgen, $16$ games: final evaluation return on the $200$ pinned training
levels ($25$M steps; mean$\,\pm\,$1 s.d.\ over $10$ seeds per method). Normalized training returns per run are computed by dividing the average test return per run for each environment by the corresponding
average test return of the shared critic baseline over all runs. \textbf{Bold} $=$ best in the row.}
\label{tab:popart}
\setlength{\tabcolsep}{4pt}
\renewcommand{\arraystretch}{1.08}
\begin{tabular}{l rrrr}
\toprule
& \multicolumn{4}{c}{Final evaluation return} \\
\cmidrule(lr){2-5}
Game & shared & FiLM & multihead & \,+\,PopArt \\
\midrule
Bigfish & 5.89\,\tiny$\pm$0.64 & 6.77\,\tiny$\pm$0.58 & 7.36\,\tiny$\pm$0.69 & \textbf{7.44\,\tiny$\pm$1.86} \\
Bossfight & 5.11\,\tiny$\pm$3.56 & 8.86\,\tiny$\pm$0.27 & \textbf{10.13\,\tiny$\pm$0.24} & 9.11\,\tiny$\pm$0.59 \\
Caveflyer & 3.11\,\tiny$\pm$0.17 & 3.83\,\tiny$\pm$0.21 & \textbf{4.18\,\tiny$\pm$0.21} & 3.03\,\tiny$\pm$0.39 \\
Chaser & 1.27\,\tiny$\pm$0.21 & 1.43\,\tiny$\pm$0.17 & \textbf{1.96\,\tiny$\pm$0.15} & 1.33\,\tiny$\pm$0.16 \\
Climber & 3.87\,\tiny$\pm$0.27 & 4.39\,\tiny$\pm$0.15 & \textbf{5.29\,\tiny$\pm$0.18} & 3.51\,\tiny$\pm$0.34 \\
Coinrun & 6.83\,\tiny$\pm$0.37 & 7.66\,\tiny$\pm$0.20 & \textbf{7.86\,\tiny$\pm$0.10} & 7.50\,\tiny$\pm$0.36 \\
Dodgeball & 3.08\,\tiny$\pm$0.20 & 3.42\,\tiny$\pm$0.15 & \textbf{3.66\,\tiny$\pm$0.12} & 3.57\,\tiny$\pm$0.25 \\
Fruitbot & 13.35\,\tiny$\pm$1.01 & \textbf{13.70\,\tiny$\pm$0.32} & 13.34\,\tiny$\pm$0.45 & 13.66\,\tiny$\pm$0.62 \\
Heist & 4.14\,\tiny$\pm$0.30 & \textbf{4.27\,\tiny$\pm$0.22} & 4.12\,\tiny$\pm$0.26 & 3.29\,\tiny$\pm$0.38 \\
Jumper & 5.44\,\tiny$\pm$0.22 & \textbf{5.75\,\tiny$\pm$0.12} & 5.41\,\tiny$\pm$0.10 & 3.88\,\tiny$\pm$1.11 \\
Leaper & 3.26\,\tiny$\pm$0.26 & 3.61\,\tiny$\pm$0.31 & \textbf{4.08\,\tiny$\pm$0.37} & 3.59\,\tiny$\pm$0.20 \\
Maze & 5.21\,\tiny$\pm$0.23 & 5.45\,\tiny$\pm$0.25 & \textbf{5.96\,\tiny$\pm$0.15} & 5.08\,\tiny$\pm$0.32 \\
Miner & 1.41\,\tiny$\pm$0.25 & 1.94\,\tiny$\pm$0.32 & \textbf{2.11\,\tiny$\pm$0.26} & 2.07\,\tiny$\pm$0.35 \\
Ninja & 5.71\,\tiny$\pm$0.30 & 6.67\,\tiny$\pm$0.16 & \textbf{6.73\,\tiny$\pm$0.31} & 6.73\,\tiny$\pm$0.68 \\
Plunder & 3.46\,\tiny$\pm$0.60 & \textbf{3.67\,\tiny$\pm$0.38} & 2.63\,\tiny$\pm$0.29 & 2.70\,\tiny$\pm$0.38 \\
Starpilot & 14.44\,\tiny$\pm$1.50 & 16.61\,\tiny$\pm$1.10 & 17.77\,\tiny$\pm$2.04 & \textbf{19.16\,\tiny$\pm$1.69} \\
\midrule
\textbf{Normalized training return (\%)} & 100.0\,\tiny$\pm$19.0 & 116.5\,\tiny$\pm$19.0 & \textbf{124.2\,\tiny$\pm$28.3} & 110.0\,\tiny$\pm$29.7 \\
\bottomrule
\end{tabular}
\end{table}

\clearpage


\begin{center}
\begin{minipage}{\linewidth}
\centering
\includegraphics[width=\textwidth]{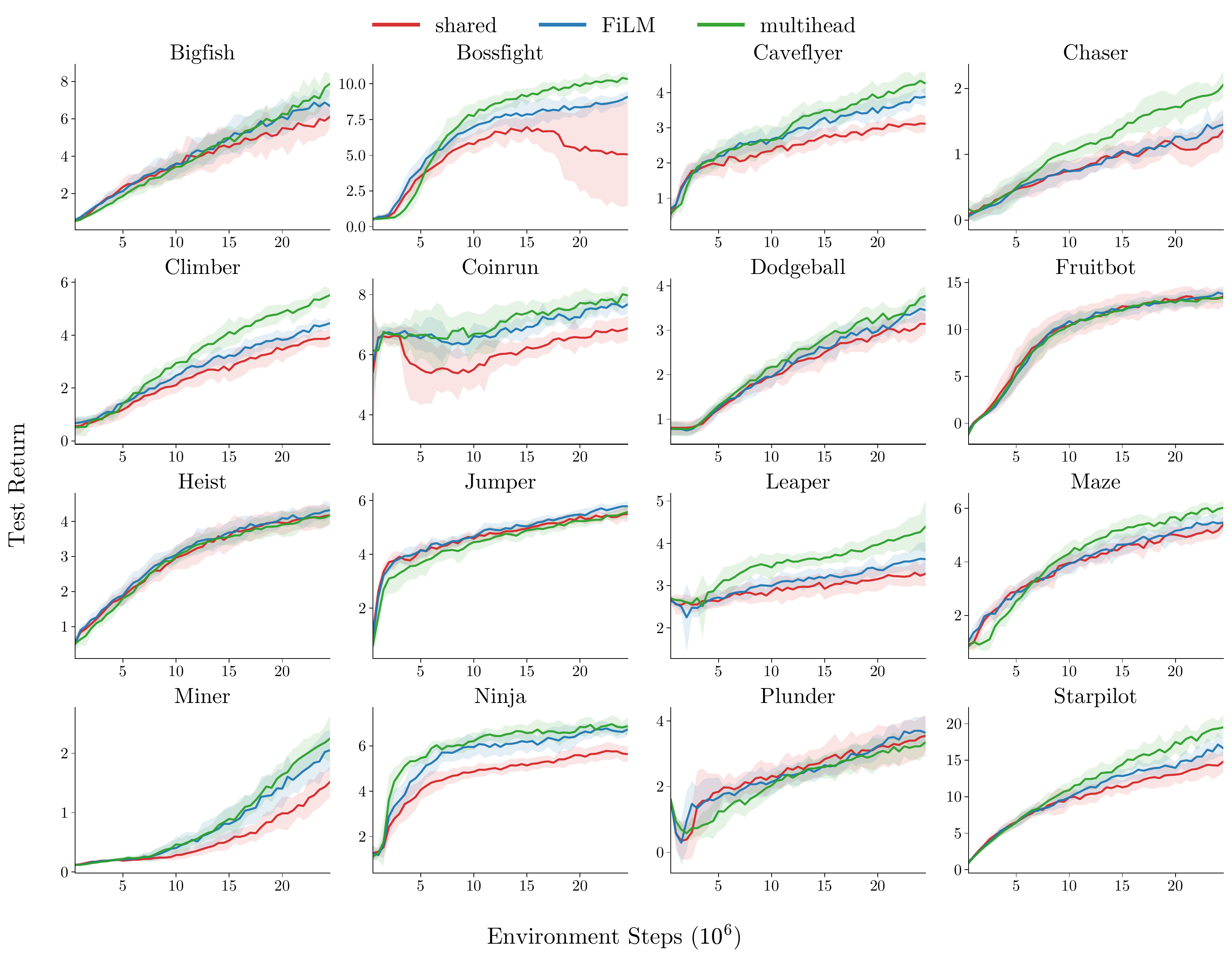}
\captionof{figure}{Complete Procgen learning curves for all $16$ games. Evaluation return on the
$200$ pinned training levels over $25$M environment steps, shown as mean$\,\pm\,$1 s.d.\
over $10$ seeds for the shared, FiLM, and multihead critics.
Figure~\ref{fig:procgen}
shows the subset of eight games in the main text.}
\label{fig:procgen-full}
\end{minipage}
\end{center}

\subsection{Value loss diagnostics}
\label{app:value-loss-diagnostics}

We report the clipped PPO value objective throughout training to test whether critic
conditioning creates a sustained fitting cost beyond the early CartPole transient.
At each minibatch this is one half of the larger of the two squared errors from the
unclipped and clipped value predictions, averaged over PPO epochs and minibatches. The loss is measured
on each method's own training stream and therefore diagnoses the
optimization burden encountered by that method. Across the larger benchmarks,
conditioning does not impose a systematic penalty: the early multihead excess closes on
BipedalWalker, conditioned losses are lower through most of training on all three MuJoCo
tasks, and their ordering is game dependent but often favorable on Procgen.

\begin{center}
\begin{minipage}{\linewidth}
\centering
\includegraphics[width=\textwidth]{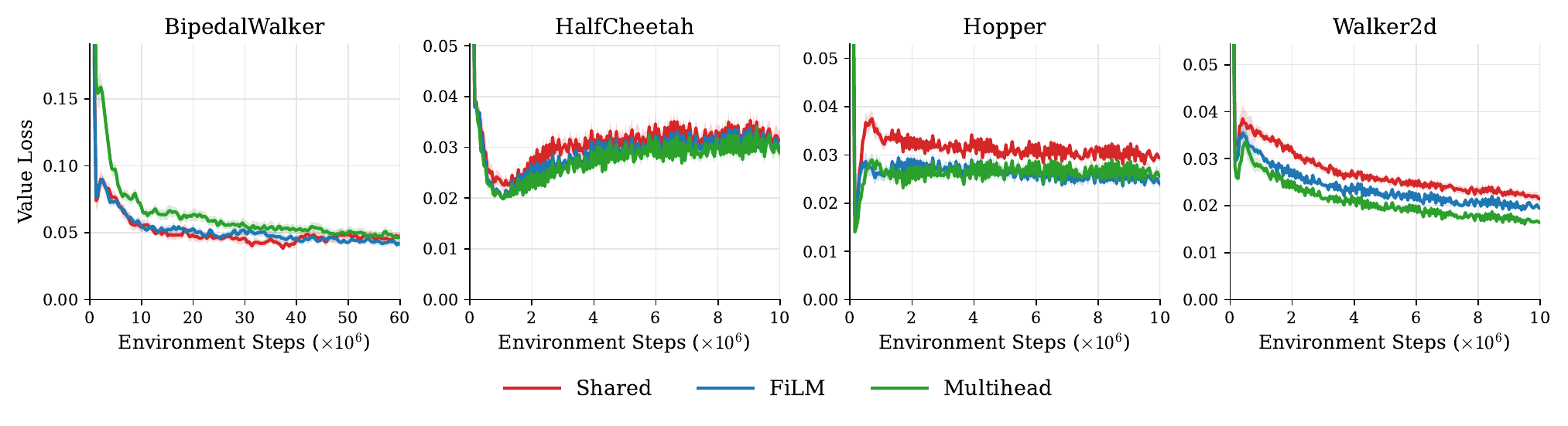}
\captionof{figure}{\textbf{Value loss diagnostics beyond CartPole.}
The clipped PPO value objective on BipedalWalker and the three MuJoCo tasks; each panel
uses its own vertical scale. On BipedalWalker, FiLM reaches a slightly lower loss late in training
and the early multihead excess closes to a comparable level. On HalfCheetah, Hopper, and
Walker2d, the conditioned critics remain below the shared critic through most of
training. Curves show mean$\,\pm\,$1 s.d.\ over $10$ seeds on BipedalWalker and
$30$ seeds for each MuJoCo task.}
\label{fig:value-loss-continuous}
\end{minipage}
\end{center}

\begin{center}
\begin{minipage}{\linewidth}
\centering
\includegraphics[width=\textwidth]{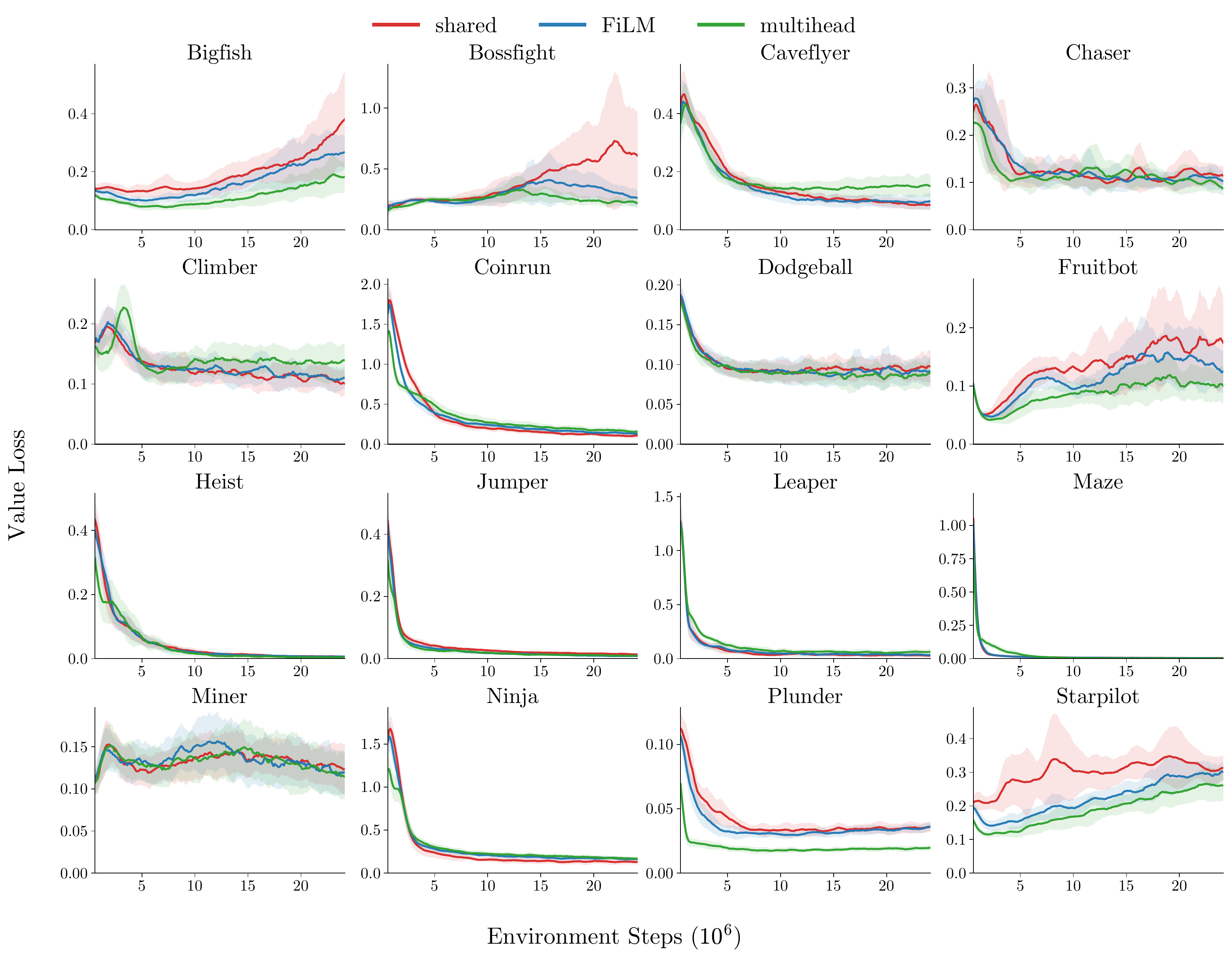}
\captionof{figure}{\textbf{Procgen value loss trajectories.}
The clipped PPO value objective over $25$M environment steps for all $16$ games; each
panel uses its own vertical scale. In late training, both the FiLM and multihead mean
curves lie below the mean for the shared critic in $11$ games. Curves show mean$\,\pm\,$1 s.d.\
over $10$ seeds. The ordering varies by task, and conditioning does not increase the loss
uniformly or create a systematic bottleneck in value fitting.}
\label{fig:value-loss-procgen}
\end{minipage}
\end{center}

\end{document}